\documentclass{article} 
\usepackage{iclr2026_conference,times}

\usepackage{amsmath,amsfonts,bm}

\def\eqref#1{equation~\ref{#1}}

\def\1{\bm{1}}

\def\vk{{\bm{k}}}

\def\vv{{\bm{v}}}

\DeclareMathAlphabet{\mathsfit}{\encodingdefault}{\sfdefault}{m}{sl}
\SetMathAlphabet{\mathsfit}{bold}{\encodingdefault}{\sfdefault}{bx}{n}

\newcommand{\E}{\mathbb{E}}

\newcommand{\R}{\mathbb{R}}

\usepackage{hyperref}
\usepackage{url}
\usepackage{graphicx}
\usepackage{booktabs}
\usepackage{multirow}
\usepackage{multicol}
\usepackage{wrapfig}
\usepackage{float}
\usepackage{listings}
\usepackage{xcolor}

\usepackage{amsthm}

\newtheorem{theorem}{Theorem}
\newtheorem{lemma}{Lemma}

\newtheorem{remark}{Remark}

\usepackage{amsmath}

\newcommand{\uu}{\boldsymbol{u}}
\let\vv\undefined
\newcommand{\vv}{\boldsymbol{v}}
\let\vk\undefined
\newcommand{\vk}{\boldsymbol{v}^{(k)}}
\let\X\undefined
\newcommand{\X}{\boldsymbol{X}}
\newcommand{\x}{\boldsymbol{x}}
\newcommand{\I}{\boldsymbol{I}}
\newcommand{\Rm}{\boldsymbol{R}}
\newcommand{\G}{\boldsymbol{G}}

\let\E\undefined
\newcommand{\E}{\boldsymbol{E}}
\let\r\undefined
\newcommand{\r}{\mathcal{R}}

\newtheorem{innercustomthm}{Theorem}
\newenvironment{customthm}[1]
  {\renewcommand\theinnercustomthm{#1}\innercustomthm}
  {\endinnercustomthm}

\newtheorem{innercustomlemma}{Lemma}
\newenvironment{customlemma}[1]
  {\renewcommand\theinnercustomlemma{#1}\innercustomlemma}
  {\endinnercustomlemma}

\title{Frayed RoPE and Long Inputs:\\A Geometric Perspective}

\author{Davis Wertheimer, Aozhong Zhang*, Derrick Liu, Penghang Yin*, Naigang Wang\\
IBM Research; *University at Albany, SUNY\\
\texttt{\{davis.wertheimer,dliu\}@ibm.com, nwang@us.ibm.com}\\
*\texttt{\{azhang3,pyin\}@albany.edu}\\
}

\newcommand\sref{\S~\ref}

\iclrfinalcopy 
\begin{document}

\maketitle

\begin{abstract}
Rotary Positional Embedding (RoPE) is a widely adopted technique for encoding position in language models, which, while effective, causes performance breakdown when input length exceeds training length. Prior analyses assert (rightly) that long inputs cause channels to rotate ``out of distribution,'' but it is not clear how extra rotation relates to or causes pathological behavior. Through empirical and theoretical analysis we advance a unified geometric understanding of attention behavior with RoPE. We find that attention induces tight clustering of separated key and query latent point clouds, allowing for creation of sink tokens: placeholders that allow attention heads to avoid token mixing when not required. RoPE applied to longer inputs damages this key/query cluster separation, producing pathological behavior by inhibiting sink token functionality. From this geometric perspective, we propose \name (In Distribution), a straightforward modification that allows attention layers to generalize to longer inputs out of the box: apply RoPE with high frequency to a subset of channels. We demonstrate the effectiveness of \name for extended inputs using 1B and 3B parameter Transformers on the LongBench and RULER information retrieval benchmarks. 
\end{abstract}

\section{Introduction}

Transformer models form the backbone of modern large language models (LLMs), enabling them to capture complex dependencies across long sequences. The attention mechanism in transformers maps inputs into queries, keys, and values: queries and keys determine token relevance through similarity scores, while values provide the content to be aggregated. This separation allows the model to learn both where to attend and what information to extract, producing context-aware representations that drive the success of transformers in natural language processing and beyond.

To enhance the interaction between queries and keys, positional encoding is used to distinguish token order, constituting a fundamental component of transformer design. 
Rotary Positional Embedding (RoPE)~\citep{su2023roformerenhancedtransformerrotary} has emerged as the predominant approach and is now implemented in most state-of-the-art LLMs, including LLaMA~\citep{grattafiori2024llama3herdmodels}, GPT~\citep{openai2025gptoss120bgptoss20bmodel}, and DeepSeek~\citep{deepseekai2025deepseekv3technicalreport}. However, a key limitation of RoPE is performance degradation when input length exceeds training context. Most attempts to analyze and address this issue attribute the failure to channels rotating ``out of distribution,'' leading to frequency rescaling as a workaround~\citep{chen2023extendingcontextwindowlarge, NTK-by-parts, NTK-Aware, peng2023yarnefficientcontextwindow,ding2024longrope}.

Another important phenomenon in transformers is the attention sink, which has been shown to influence long-context generalization~\citep{xiao2023efficient}. The attention sink, typically the first input token, possesses little semantic meaning but consistently large attention scores. Its presence is considered crucial to prevent over-mixing of information, and empirical evidence shows that attention sinks must be preserved when extending context lengths~\citep{xiao2023efficient,han2024lminfinitezeroshotextremelength}.

The relationship between attention, RoPE, and attention sinks -- three seemingly disconnected concepts -- is the focus of this paper. We propose a unified geometric perspective, based on analysis of popular LLM families including LLaMA, Gemma~\citep{team2024gemma}, and Olmo~\citep{groeneveld2024olmo}. 
We find that, contrary to the common intuition of attention as a soft nearest-neighbor lookup, queries and keys form tight clusters with minimal overlap, while the sink token resides near the origin (Fig.~\ref{fig:cones}, left). Within the training context length, this separation ensures that the sink token, with its small norm, naturally absorbs the majority of attention weight. Beyond the training context, RoPE disperses and overlaps the query/key clusters. This geometric disruption prevents the sink token from functioning, leading to long-context performance breakdowns (Fig.~\ref{fig:cones}, right).

    \begin{wrapfigure}{R}{.5\textwidth}
    \centering
    \vspace{-6pt}
    \includegraphics[width=.9\linewidth]{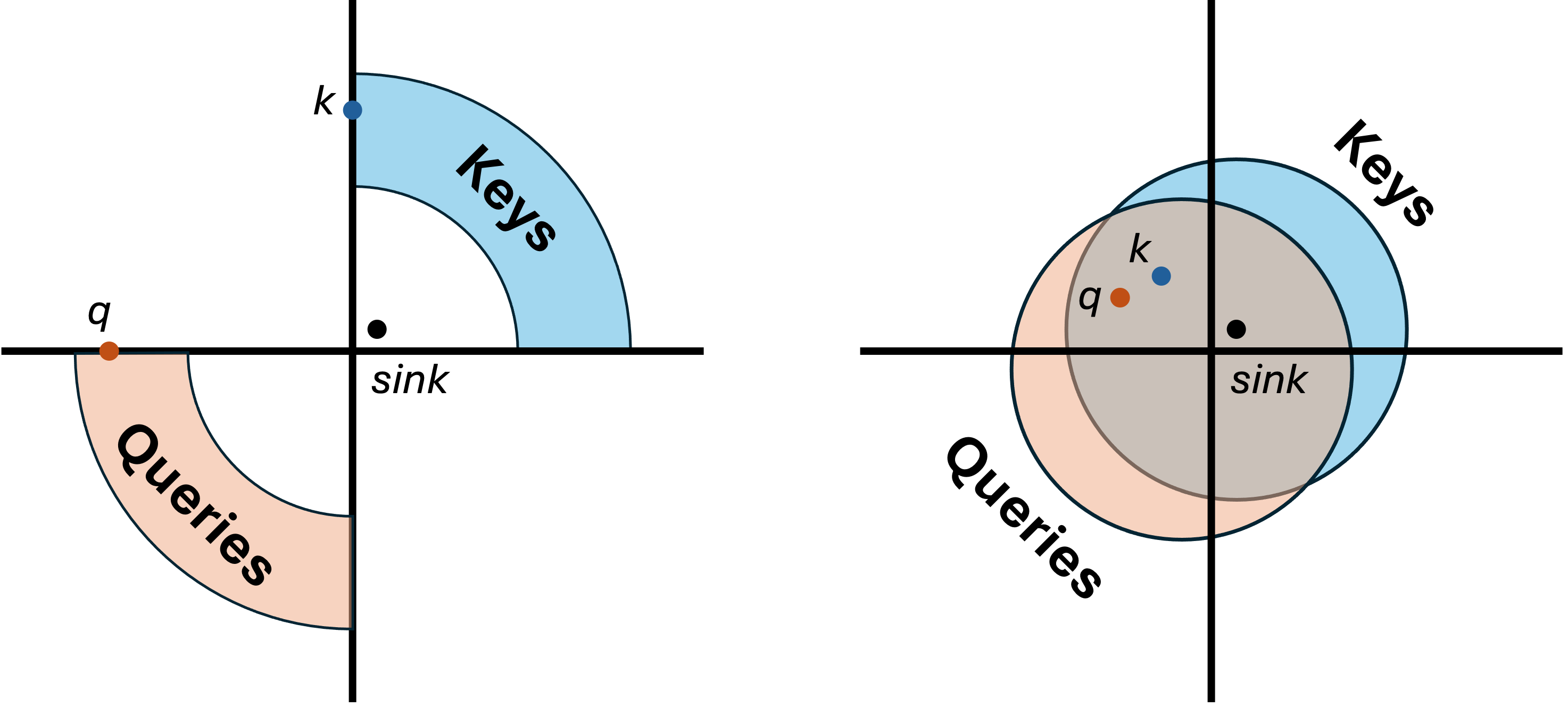}
    \vspace{-10pt}
    \caption{A 2D diagram of our observed latent geometry. \textbf{Left:} Keys/queries cluster tightly into opposing point clouds with negative dot products. The sink token has low norm and thus the greatest dot product. Matched key/query pair $q,k$ align orthogonally, letting their product approach and exceed the sink's. \textbf{Right:} RoPE on long inputs makes keys/queries disperse and overlap, causing spurious alignment. Sink token stops functioning.}
    \label{fig:cones}
    \vspace{-10pt}
    \end{wrapfigure}


Based on this analysis, we identify two necessary criteria for length generalization with RoPE: 1) a lower bound on cluster overlap, and 2) attaining this bound within the training length. The former ensures continued sink token functionality in the limit, while the latter prevents out-of-distribution errors. 
We propose \name (RoPE In Distribution) as a simple yet effective example for how to fulfill these criteria. 
A plug-in replacement for RoPE,
\name applies high-frequency rotation to a fraction of channels in each head, leaving the rest unchanged.
RoPE-free channels maintain separated key/query clusters across length (criteria 1), while high-frequency rotation ensures that in subspaces where clusters merge, they do so within the training length (criteria 2), facilitating length generalization.
We validate \name on LongBench and RULER using trained 1B- and 3B-parameter decoders, demonstrating strong context length generalization and improvements over prior tuning-free methods.





\section{Background and Related Work}

  \textbf{Position Encoding}
    techniques can be broadly divided into two categories: absolute position embeddings (APE) and relative position embeddings (RPE). APE directly injects position information into latent representations using token indices, in a fixed sinusoidal~\citep{vaswani2023attentionneed} or learnable form~\citep{devlin2019bertpretrainingdeepbidirectional}. APE exhibits limited generalization beyond the training context length, which RPE addresses by encoding distances between token pairs rather than their absolute indices. Notable RPEs include T5's relative bias~\citep{raffel2023exploringlimitstransferlearning} and Alibi~\citep{press2022trainshorttestlong}, which add distance-based biases to attention logits, and RoPE (Rotary Position Embedding)~\citep{su2023roformerenhancedtransformerrotary}, which encodes relative distance as latent angular displacement, now the de facto standard for LLMs. 

    The key insight behind RoPE is that relative position, through properties of rotation, decomposes into independent key and query transformations. RoPE encodes relative position via angular displacement proportional to token distance, interposing a block-diagonal matrix of $2\times2$ rotations into the key/query dot product. Each submatrix has a constant frequency $\theta$ scaling token distance $m$: 
    \setlength{\arraycolsep}{2pt}
    \begin{equation}    
	\mathbf{R}^m_{\Theta} = 
    \text{Diag}\left(
    \begin{bmatrix}
    \cos{m\theta_1}& -\sin{m\theta_1}\\
    \sin{m\theta_1}&\cos{m\theta_1}
    \end{bmatrix}
    ,\cdots,
    \begin{bmatrix}
    \cos{m\theta_{d/2}}& -\sin{m\theta_{d/2}}\\
    \sin{m\theta_{d/2}}&\cos{m\theta_{d/2}}
    \end{bmatrix}
    \right)
	\label{fn:rope-RMat}
    \end{equation}
    In practice, this decomposes into independent rotations on query $q_i$ and key $k_j$ at positions $i,j$:
    \begin{equation}
        \text{RoPE}(\langle q_i,k_j\rangle)
        =q_iR^{i-j}_\Theta k_j^\top
        = q_iR^i_\Theta R^{-j}_\Theta k_j^\top
        = q_iR^i_\Theta (k_jR^j_\Theta)^\top 
        =\langle r(q_i,i),r(k_j,j)\rangle 
    \end{equation}
    where $r(\cdot,i)$ represents rotation by $\mathbf{R}_\Theta^i$. Understanding the impact of this progressive rotation on latent keys and queries is crucial to understanding out-of-distribution failures on long contexts.
    
    

  \textbf{Context Length Extension:}
 As LLMs are increasingly applied to long-context tasks, substantial research has focused on extending their usable sequence lengths without retraining. While RPEs such as RoPE are designed to improve long-context generalization, performance still degrades beyond the training length. To address this, Position Interpolation (PI)~\citep{chen2023extendingcontextwindowlarge} linearly interpolates position indices within the pre-trained sequence length. NTK-by-parts~\citep{NTK-by-parts} and NTK-aware~\citep{NTK-Aware} introduce nonlinear interpolation schemes inspired by Neural Tangent Kernel dynamics, 
 scaling RoPE frequencies in groups and outperforming simple PI.
 YaRN (Yet Another RoPE Extension)~\citep{peng2023yarnefficientcontextwindow} further integrates previous NTK-based approaches with temperature scaling on attention logits.
 More recently, LongRoPE~\citep{ding2024longrope} employs evolutionary search to optimize the frequency rescale factors for each dimension. LongRoPE extends context window to beyond 2 million tokens, albeit with multi-step long-context fine-tuning. Here we focus on tuning-free generalization and leave tuning to future work.

 Some studies find that RoPE's low-frequency components induce high-norm semantic bands, which become unstable in long contexts~\citep{barbero2024round} or hinder the encoding of semantic information~\citep{chen2024hopenovelpositionalencoding}. They propose limiting RoPE to a subset of channels, finding this improves performance. Our analysis provides novel perspective and caveats for this technique.

  \textbf{Sink Tokens},
    or attention sinks, refer to tokens with disproportionately high attention despite a lack of semantic meaning~\citep{xiao2023efficient}. This phenomenon is widely observed in LLMs and plays a critical role in preserving model behavior, especially in long contexts~\citep{xiao2023efficient,han2024lminfinitezeroshotextremelength, yang2024seedstorymultimodallongstory}. 
    Usually the first token of a sequence
    \citep{xiao2023efficient, cancedda2024spectralfiltersdarksignals}, attention sinks have also been linked to massive activations observed across LLMs~\citep{sun2024massiveactivationslargelanguage}. \cite{gu2025attentionsinkemergeslanguage} investigate when and how attention sinks emerge during pretraining, and \cite{barbero2025llms} provide theoretical and empirical evidence that sink tokens prevent representation collapse from information over-mixing~\citep{barbero2024transformersneedglassesinformation}. Here we relate sink tokens to RoPE and attention geometry, pinpointing the sink token as the failure mechanism in long contexts.
  
\section{Analysis}
\label{sec:analysis}
  We perform a geometric analysis of attention with RoPE, showing that keys and queries cluster tightly in opposing directions, while RoPE inhibits this behavior, with clusters dispersing and overlapping over time.
  Alongside small sink token $\ell_2$ norm, these separated clusters produce a learned bias toward the sink. However, as RoPE disperses and overlaps key and query clusters, this mechanism becomes tenuous. We claim that the breakdown of transformers in long contexts is a breakdown of the sink token: past the training length, models begin attending to the wrong token(s) by default.

  Analysis is conducted on Llama3-8B-Instruct, with additional trials on Olmo-7b and Gemma-7b for verification. When not otherwise specified, the relevant model is Llama3. Input text is drawn from the Wikitext2 dataset~\citep{merity2016pointer}. Further details are provided in \ref{app:details}.
  
  \subsection{Key/Query Clustering}
  \label{sec:clusters}
    An intuitive understanding of the attention operation is that it functions as a soft nearest-neighbor lookup. A query is oriented in latent space to align with one or more contextually relevant key vectors, and the degree of alignment defines the mixing ratios for corresponding values. 
    The curse of dimensionality ensures that random IID latent points are orthogonal by default, so directional alignment in high-dimensional space is difficult. Thus we can imagine that keys and queries form overlapping point clouds around the origin. Key/query matching is accomplished by high directional alignment: activated pairs should have large, positive dot products. Keys and non-matching queries, meanwhile, should be orthogonal, with small dot products, to keep retrieval discriminative. 

    This model of attention, while intuitive, is also wrong, at least for RoPE models. Instead of overlapping clouds on the origin, keys and queries form tight clusters away from the origin, with minimal overlap. Further, such clusters are generally \textit{un}aligned directionally, with the origin sitting between the clusters. Fig.~\ref{fig:angles} shows the mean intra- and inter-cluster pairwise cosine distances ($\ell_2$-normalized dot products) for keys and queries, averaged over layers and heads. Before RoPE, intra-cluster distances (key-key and query-query), bounded to $\pm1$, are generally close to 1. Key and query point clouds are situated within a tight arc, in small clusters displaced from the origin.
    At the same time, these clusters are largely aligned \textit{against} each other, with negative mean key-query dot product in Fig.~\ref{fig:angles} (right). This paints an entirely different picture of attention: instead of overlapped point clouds, envision keys and queries in opposing quadrants. Queries avoid attending to most keys via negative dot products. If a query and key land on aligned quadrant boundaries, though, the resulting zero dot product exceeds the baseline and yields a large attention weight. Softmax shift-invariance ensures that this arrangement (negative baseline and orthogonal ``aligned'' pairs) produces identical mixing behavior to the original conception (orthogonal baseline and positive products for aligned pairs). Fig.~\ref{fig:cones} (left) illustrates the proposed geometry. In practice, key/query clusters are not \textit{directly} across the origin (dot products in Fig.~\ref{fig:angles} (right) are negative but small), but the intuition holds. 

    \begin{figure}[h!]
    \centering
    \includegraphics[width=.9\textwidth]{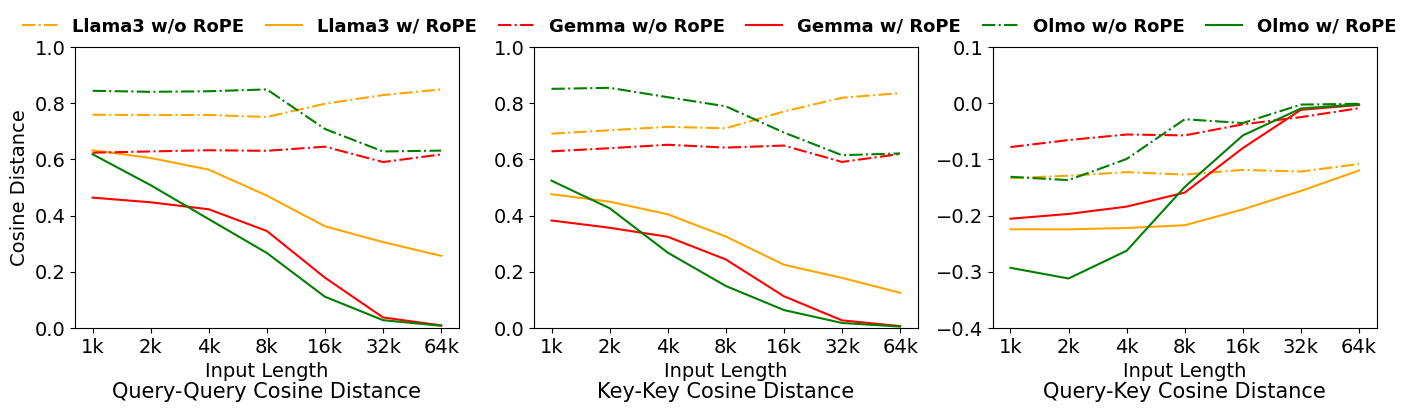}
    \vspace{-10pt}
    \caption{Effect of RoPE across context length on pairwise angular distances within heads for Llama3-8B, Gemma-7B and OLMo-7B. 
    }
    \label{fig:angles}
    \end{figure}

    \subsubsection{The Impact of RoPE on Clustering}
    RoPE sends points through a predetermined trajectory of orbits around the origin, so any tightly-clustered point cloud, displaced from the origin, should inevitably disperse. This is indeed the case for transformers: Fig.~\ref{fig:angles} shows that RoPE decreases intra-cluster alignment, with further decrease over time as points spiral further away.
    The model compensates for this by positioning key/query clusters such that RoPE misaligns them \textit{further}: key-query dot product \textit{also} decreases, and does not rise until after the training length (2k for Olmo, 8k Llama/Gemma). Meanwhile, the clusters without RoPE mostly maintain their behavior across context lengths. 
    It thus appears that RoPE weakens the clustering, but does not eliminate it until the training context length is exceeded.
    We corroborate this finding via additional distance/clustering metrics in \ref{app:clusteringmetrics}.

    Fig.~\ref{fig:pca} illustrates this behavior, by taking a PCA ``snapshot'' of the point clouds without RoPE at time $t=4096$, and applying the same projection with RoPE and at time $t=65536$ (more plots are available in \ref{app:morepca}). Prior observations are confirmed visually: in the first and third views, points form tight clusters displaced from the origin, and key and query clusters are located across from each other (with four queries per key cloud due to GQA~\citep{ainslie2023gqa}). RoPE causes clusters within training length to disperse slightly, but at length 64k eliminates cluster separation entirely (in this projection). The exact impact of this overlap in key/query clouds is discussed in \sref{sec:sinks}, but it is clear that query-key ``alignment'' works differently at length 64k compared to the other scenarios.

    \begin{figure}[h!]
    \centering
    \includegraphics[width=.9\textwidth]{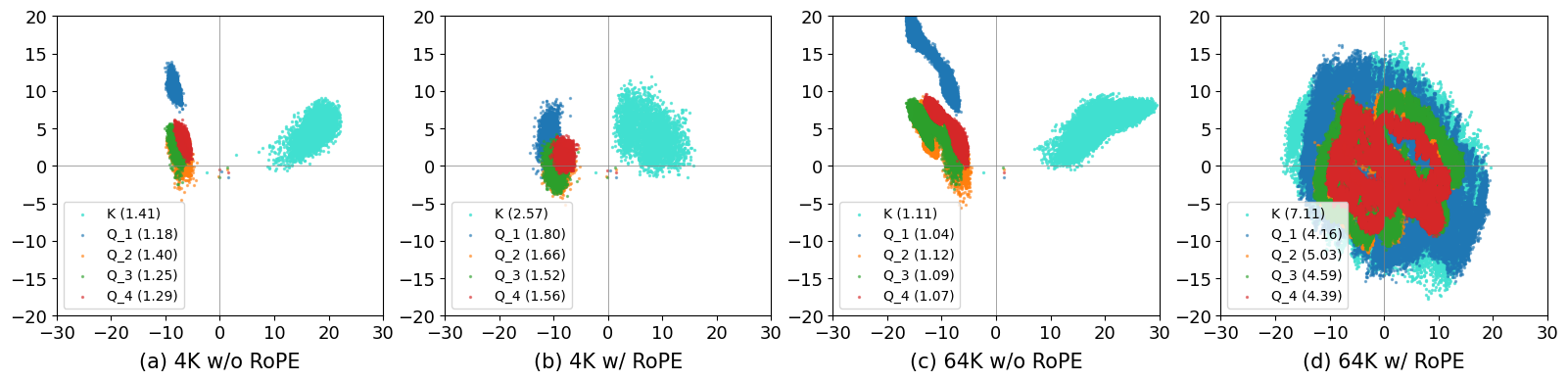}
    \vspace{-10pt}
    \caption{2D PCA projections of Llama3 representations under different context lengths and RoPE settings (3rd key head of layer 21 and its queries). RoPE at long contexts destroys cluster separation, and increases stable rank (in parentheses) with sequence length, consistent with Theorem~\ref{thm:stable_rank}.
    }
    \label{fig:pca}
    \end{figure}
    \subsubsection{A Singular Value Perspective on Clustering}
    While the visual analysis is striking, it only captures a 2D projection of a 128-dimensional latent space. Similarly, Fig.~\ref{fig:angles} reports pairwise relationships, an incomplete picture of global behavior. 
    We therefore corroborate our findings with a holistic and rigorous analysis based on singular values. 
    In an attention head, the set of key or query vectors forms a matrix $\X\in\R^{n\times d}$, where $n$ is the inference sequence length and $d$ is head width. $\X$'s singular values correspond to the principal components of the point cloud, an ordered set of directions maximizing variance along the earliest directions up to orthogonality constraints. When singular values are equal, variance is constant in all directions, and the point cloud forms a ball around the origin. Unequal values indicate uneven spread. In practice, the first singular value (FSV) of key and query clusters before RoPE is large, accounting for over $75\%$ of total cluster variance on average for Llama3. Fig.~\ref{fig:singulars} (left) plots the distribution across individual heads and layers, and the degree of variance covered by the first principal component ranges from about half to nearly all of it. Thus the majority of spread around the origin occurs in one direction: either the cluster is a long thin needle, or it is a tight cloud displaced along this direction. Given the intra-cluster dot products in Fig.~\ref{fig:angles}, the latter must be the case.

    \begin{figure}[h!]
    \centering
    \includegraphics[width=.9\textwidth]{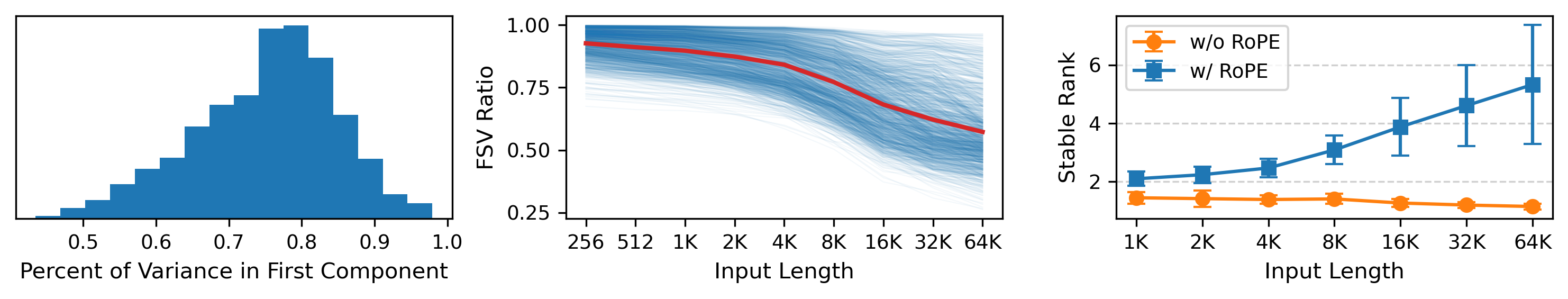}
    \vspace{-10pt}
    \caption{\textbf{Left}: Histogram across layers and heads showing the percentage of variance (relative to origin) explained by the first principal component of latent key/query clusters. 
    \textbf{Middle}: Ratio of first singular value (FSV) after and before RoPE. Blue lines plot individual key/query heads, red plots the average trend. RoPE shrinks the FSV, but accelerates beyond the training length, producing cluster dispersal.
    \textbf{Right}: Mean stable rank of key clusters by input length, showing monotonic increase with RoPE, indicating cluster dispersal. Error bars denote standard deviation.
    }
    \label{fig:singulars}
    \end{figure}

    When RoPE is applied, we expect principal components to skew more evenly: RoPE throws channel pairs through decorrelated rotations (ensured by irrational frequency ratios), so in the limit, a point cloud under RoPE maps to a shell of IID points orbiting the origin. We can formalize and justify this intuition: FSV (i.e., the spectral norm) ratio under RoPE, $\frac{\|\r(\X)\|_2}{\|\X\|_2}$, shrinks as the sequence length $n$ increases, where $\mathcal{R}$ denotes the RoPE action. The following informal result provides an asymptotic estimate of this ratio; the rigorous statement and proof are deferred to \ref{app:theory}.
    


\begin{lemma}\label{lem:spectral}
Assume the key/query matrix $\X = \uu\vv^\top$, where $\uu\in\R^n$, $\vv\in\R^d$ with $\|\vv\|_2 = 1$. Then as the inference sequence length $n\to \infty$, under mild conditions on $\uu$, applying RoPE to $\X$ shrinks its spectral norm (or FSV) by at least a factor of $\sqrt{2}$, and by up to a factor of $\sqrt{d}$.
\end{lemma}

Fig.~\ref{fig:singulars} (left) shows that key/query matrices are close to rank-1 in practice, which motivates our rank-1 assumption on 
$\X$ in Lemma~\ref{lem:spectral}. Under this assumption, the $j^{\mathrm{th}}$ key/query projection takes the form $u_j \vv \in \mathbb{R}^d$.
We expect a similar result to hold when this assumption is relaxed to the approximately rank-1 case, which we leave for future work. The exact reduction factor depends on the energy distribution across the channels of the underlying learned representation $\vv$, as detailed in \ref{app:theory}.

Fig.~\ref{fig:singulars} (middle) plots the ratio of FSV under RoPE in practice. In all cases, RoPE decreases the FSV as expected, but the decrease is limited within the training context, falling smoothly up to length 4k. The decrease is more aggressive for inputs above the 8k training length. This confirms that cluster means drift toward the origin as RoPE is applied over longer sequences. 

At the same time, RoPE preserves the sum of $\X$'s squared singular values, or equivalently, preserving the Frobenius norm, $\|\X\|_F$ \citep{golub2013matrix}:

\begin{lemma}\label{lem:fro}
Applying RoPE preserves the Frobenius norm of any $\X$, i.e., $\|\r(\X)\|_F = \|\X\|_F$.
\end{lemma}

Therefore, as the FSV falls due to RoPE, other singular values (representing other principal components) must grow to compensate. 
This suggests that RoPE induces an expansion and dispersion of clusters while pulling them toward the origin. We formalize this intuition and characterize cluster dispersion using the notion of \textbf{stable rank}~\citep{rudelson2007sampling,sanyal2019stable}, a continuous analogue of standard matrix rank. For a nonzero matrix $\X\in\R^{n\times d}$ with $n\geq d$, its stable rank is defined as 
$$\mathrm{srank}(\X) := \frac{\|\X\|_F^2}{\|\X\|_2^2} = \frac{\sum_{i=1}^d \sigma_i^2}{\sigma_1^2},$$
where $\sigma_1\geq \dots \geq \sigma_d$ are $\X$'s singular values, with the following properties: (i) $1\leq\mathrm{srank}(\X) \leq \mathrm{rank}(\X) \leq d$; (ii) $\mathrm{srank}(\X) = 1 \Leftrightarrow \mathrm{rank}(\X) = 1$; (iii) $\mathrm{srank}(\X) = d \Leftrightarrow \sigma_i$'s are all equal.

Theorem~\ref{thm:stable_rank}, which follows directly from Lemmas~\ref{lem:spectral} and~\ref{lem:fro}, shows that applying RoPE increases the stable rank when the input length $n$ becomes large, as demonstrated by Fig.~\ref{fig:singulars} (right). This increase in stable rank causes the point cloud to spread out around the origin, as illustrated in the last panel of Fig.~\ref{fig:pca}, where the corresponding stable rank values are reported.

\begin{theorem}\label{thm:stable_rank}
    Under the same assumptions as in Lemma \ref{lem:spectral},
applying RoPE to $\X$ increases its stable rank by at least a factor of 2, and by up to a factor of 
$d$, as the sequence length $n\to \infty$.
\end{theorem}
Similar to Lemma~\ref{lem:spectral}, the precise increase factor depends on the energy distribution of $\vv$ across the rotation planes: concentration on a single plane yields a factor of $2$, whereas uniform distribution across all planes yields a factor of $d$. 

We conclude that our prior observations accurately describe high-dimensional clustering behavior. 


  \subsection{Sink Tokens}
  \label{sec:sinks}
    Latent keys and queries cluster tightly into unaligned point clouds, and RoPE causes those clouds to disperse and overlap over time, particularly when input length exceeds training length. But how does this produce out-of-distribution behavior in the attention mechanism itself, and for a transformer model as a whole? We claim that the effect of clustering is mediated by the sink token. 
    
    Prior work establishes that transformers attend heavily to the first token, regardless of input~\citep{xiao2023efficient}. Prevailing wisdom is that this prevents over-mixing: information retrieval is not always useful, so attention heads must formulate a null operation~\citep{barbero2025llms}. Softmax normalizes attention scores to sum to 1, so heads cannot avoid attending. Instead, they ``sink'' attention into a placeholder key conveying no information -- in practice the first token, typically a beginning-of-sequence indicator. When an attention head does perform meaningful information retrieval, it borrows weight from the sink token and reallocates it to the chosen key, as shown in \ref{app:masks}.

    Under the intuitive overlapping-clouds model of latent keys/queries, sink token behavior is hard to reconcile. How can a single embedding align to \textit{all} directions by default? Why not default to the current token -- where keys and queries, projections of the same input, are easy to align -- as in Mamba and other linear attention layers?~\citep{gu2023mamba, katharopoulos2020transformers} Our observations help to explain this behavior. 
    Same-token key-query self-alignment is difficult to consistently impose when keys and queries are distant. It is easier to assign a sink, and place its key near the origin, granting it a near-zero dot product with all queries. When average key-query product is negative, the sink becomes the most-attended by default. This is borne out in practice: key vectors for the first input token have unusually small $\ell_2$ norm, as shown in Fig.~\ref{fig:keynorms} (left). Smaller norm produces larger dot products, as shown in Fig.~\ref{fig:keynorms} (right), which displays the average dot product of each key across subsequent queries, normalized by the largest product in the set. In this case, that's consistently the sink token, with recency bias responsible for the later upward trend. 
    
    \begin{figure}[h!]
    \centering
    \includegraphics[width=.9\textwidth]{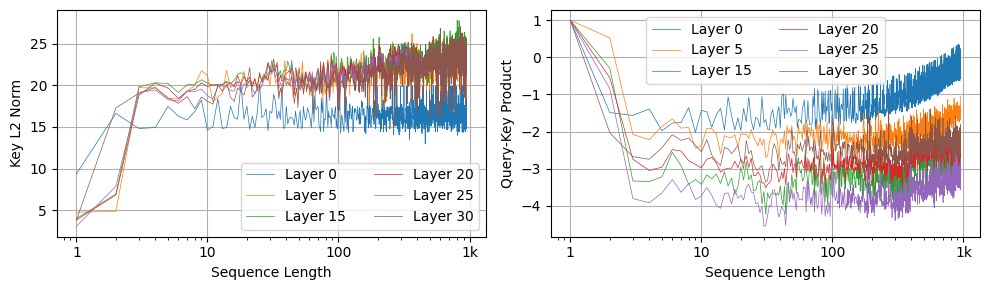}
    \vspace{-12pt}
    \caption{\textbf{Left:} Key $\ell_2$ norm as a function of position in sequence. Sink token is consistently small. \textbf{Right:} Keys have low dot product against subsequent queries in expectation, except for the first and most recent tokens. Scores are normalized by the highest value, in this case always the sink.
        }
    \label{fig:keynorms}
    \end{figure}

    \subsubsection{The Impact of RoPE on Sink Tokens}
    Within the training context length, key and query clusters are separated to the point that a sink token with small key norm can absorb the majority of attention weight. Beyond the training length, however, RoPE causes clusters to disperse and overlap. When this happens, key and query points begin to obtain positive dot products. This stops the sink token from functioning, and we claim this is the cause of out-of-distribution behavior when transformers are exposed to long inputs. 

    Fig.~\ref{fig:unsink} illustrates this behavior in Llama3. The left-hand plot captures the attention weight allocated to the sink token as a function of input length. The share of attention weight (with RoPE applied) varies widely but stably within the training length of 8k, but then falls sharply, decaying to zero over time as clusters progressively overlap and highly-aligned point pairs accumulate. Meanwhile, the activation of the sink token without RoPE stays roughly constant. We hypothesize that the activation starts off lower within the training length due to the observation in Fig.~\ref{fig:angles} (right): key/query clouds are more opposed \textit{after} applying RoPE, lowering the average dot product relative to the sink. In any case, it is clear that RoPE causes catastrophic collapse of sink token weight for long inputs. 
    
    Fig.~\ref{fig:unsink} (right) confirms that the decay in sink token attention weight is a function of key/query cluster overlap. As two point clouds approach the origin and disperse, the chance for high directional alignment between point pairs increases, and so the maximum dot pairwise product should increase as well. In practice, the maximum key/query dot product across all keys, per-query, does rise steadily over sequence length when RoPE is applied. Without RoPE, cluster behavior is stable over time, and so too, therefore, is the maximum degree of alignment between key and query points. 
    \subsubsection{A Unified Understanding of RoPE Attention}
    We now establish a unified geometric understanding of attention, RoPE, and sink tokens. The shift-invariance of softmax induces keys and queries to gather into opposing clouds across the origin. Sink tokens are implemented by positioning the first key near the origin, making that key's dot product with any query small. Because the clusters are opposed, average key/query dot product is negative, defaulting attention to the sink and mixing tokens only in the case of particularly aligned pairs. 

    RoPE complicates this arrangement by spinning points around and across the origin. Some channel pairs rotate much faster than others, but over time more and more channel pairs drift meaningfully from their original locations. 
    Eventually all channels shift into orbit, transforming previously well-separated key and query clusters into dispersed, overlapping balls. This produces positive dot products between keys and queries, overwhelming the small, but previously dominant sink token logit. Transformers with RoPE fail on long inputs because they effectively lose access to the sink token, broadcasting an excess of information from the wrong tokens forward through time. 

    \begin{figure}[h!]
      \centering
      \includegraphics[width=.9\textwidth]{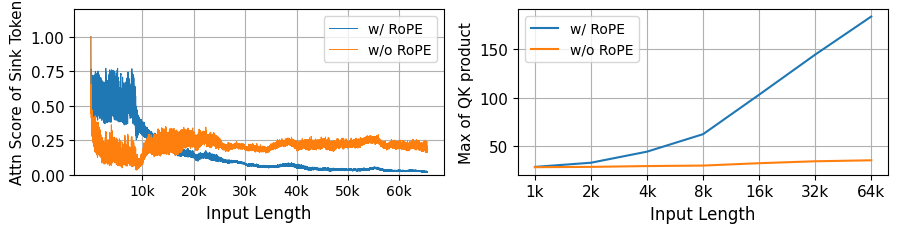}
      \vspace{-12pt}
      \caption{\textbf{Left:} Sink token attention  weight vs. input length. With RoPE, sink-token attention decays to zero beyond the training length; without RoPE it remains stable. \textbf{Right:} Maximum Query–Key dot product vs. input length. The max QK product increases with length only when RoPE is active.}
    \label{fig:unsink}
    \end{figure}
    
\section{Method}
\label{sec:method}
  Insight into RoPE's out-of-distribution behavior informs mitigation techniques. 
  Based on our analysis, two criteria must be met to preserve sink token functionality across longer sequences with RoPE. First, key and query clusters must maintain some degree of separation in the limit, in order to prevent spurious positive dot products between keys and queries that overwhelm the sink token. Second, this lower bound on cluster separation must be attained within the training length; otherwise, keys and queries will drift closer than seen during training, again opening the door to spurious positive dot products. 
  Many scaling techniques exist that (perhaps unintentionally) 
  fulfill these criteria.
  PI \citep{chen2023extendingcontextwindowlarge} and YaRN \citep{peng2023yarnefficientcontextwindow}, for example, both limit the degree of drift in low-frequency channels to that seen during training. 
  This maintains the separation of key and query clusters in those channels over extended contexts, while avoiding out-of-distribution drift.

  Training a new model from scratch offers greater design freedom. \cite{barbero2024round} suggest that information is better preserved when RoPE is limited to a fraction of channels, as positional and contextual information can be embedded in the RoPE subspace, while long-term semantic content can be allocated to the stable RoPE-free subspace. However, we still expect this approach to fail on extended contexts, as it only fulfills the first criterion (cluster stability) without the second (lower-bound attainment). Low-frequency channels still exist, reproducing the issues observed in \sref{sec:analysis}. Fig.~\ref{fig:synthetics} (left) illustrates this behavior by repeating the singular value analysis from Fig.~\ref{fig:singulars} on a synthetic point cloud. Standard RoPE (blue) fulfills neither criteria, decaying FSV all the way to zero, and not within the training length. RoPE applied to half the channels (green) successfully imposes a non-trivial lower bound on FSV (and thus cluster overlap), but does not attain it within 4k steps. Thus we can still expect out-of-distribution behavior when FSV falls below levels seen in training.

  \begin{figure}[h!]
    \centering
    \includegraphics[width=.9\textwidth]{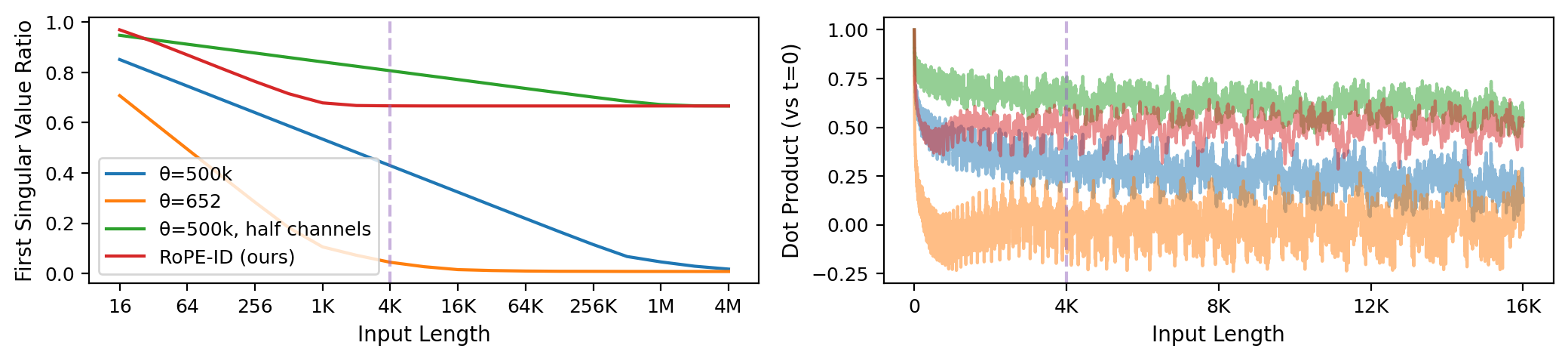}
    \vspace{-12pt}
    \caption{Expected RoPE behavior for our proposed method and three baselines. Dotted line indicates hypothetical training length of 4k and head dimension is 128. \textbf{Left:} Repeats the Fig.~\ref{fig:singulars} singular value ratio before/after RoPE, for a synthetic point cloud of ones vectors. \textbf{Right:} Long-term RoPE decay for the same techniques, showing similar behaviors. 
    }
    \label{fig:synthetics}
  \end{figure}
  
  An alternative solution is to eliminate the low frequencies of RoPE altogether: raise the lowest value enough to complete a full cycle within the training length. This ensures that clusters ``finish'' drifting to the origin, as they devolve into uncorrelated rotations -- overlapping shells around the origin -- that the model can expect to persist indefinitely. 
  Channel pairs with RoPE cannot rotate ``out of distribution'' if the entire rotation arc is covered. 
  However, this again only fulfills one of the necessary criteria for length generalization: as shown in Fig.~\ref{fig:synthetics}, high frequency RoPE (orange) with base frequency $\theta=652$ (the lowest frequency to complete a cycle in 4k steps) approaches its FSV lower bound in 4k steps -- but this lower bound is trivial. Key and query clusters are decaying entirely, and thus failing to preserve information across longer contexts.
  \cite{liu2024scaling} analyze this approach, but 
  limit their evaluation to perplexity. While perplexity is improved by the newfound stability over long contexts, it does not capture the loss of distant information through several cycles of uncorrelated rotation. 
  Long-context information retrieval is likely still difficult in this setting. 
  
  We hypothesize that \textit{combining} high RoPE frequency with partial application fulfills both criteria, yielding natural generalization to long contexts. Both changes are required for stable, discriminative behavior: key/query clusters fully merge in the RoPE subspace (satisfying criteria 2), but preserve sink token functionality via continued separation in RoPE-free subspace (criteria 1). 
  This approach, which we name \name (In Distribution),
  in Fig.~\ref{fig:synthetics} gives the best of both worlds: FSV decay (red) is non-trivially lower-bounded, and this bound is attained in 4k steps, maintaining cluster separation and sink token functionality by construction, while avoiding out-of-distribution behavior. 

  \subsection{Implementation}
  We apply RoPE to half the channels of each attention head, and adjust RoPE frequencies to attain desired behaviors. RoPE frequencies interpolate exponentially between 1 and $\frac{1}{\theta}$, where $\theta$ is the base frequency hyperparameter.
  We adjust the low end of this scale to
  two full rotations per training length, as one rotation may still preserve correlation between low frequency channels (i.e. some decay still occurs after the purple line in the orange high frequency curve of Fig.~\ref{fig:synthetics} (left)). 
  Maximum rotation speed is set to one cycle every 32 tokens, to better preserve information over short windows.
  Taking a softmax over many IID logits increases the denominator but not the numerator, resulting in the mixture distribution becoming artificially smooth over time.
  We address this via input length based temperature scaling, borrowing from \citep{peng2023yarnefficientcontextwindow}. 
  We evaluate both with and without temperature scale. 
  Ablations in \ref{app:ablations}, including comparisons to prior work of \citep{chen2024hopenovelpositionalencoding}, indicate that these simple heuristics represent safe middle-of-the-road choices, with model performance gradually improving as high frequencies are eliminated, until hitting a sharp falloff.
  Details, code and further discussion for \name can be found in \ref{app:pseudocode}.
  

  To evaluate \name, we pretrain example 1B- and 3B-parameter decoders. Models use the Llama3 tokenizer and Dolma v1.7 dataset~\citep{soldaini2024dolma}, reweighted per \citep{chu2024bamba}. Training proceeds over 21 billion tokens, with hyperparameter and architecture details provided in \ref{app:details}. 

\section{Results}
  We compare \name against four baselines. First is a vanilla decoder with RoPE base frequency $500k$, which we expect to fail beyond the training length. Second and third are approaches from \sref{sec:method}: increase RoPE frequency so that all channel pairs complete a rotation, or apply the original RoPE on half the channels per head. The former should yield stable, but poor, performance over context lengths, while the latter should mitigate but not prevent performance degradation. The strongest baseline is tuning-free extension of the vanilla model using YaRN~\citep{peng2023yarnefficientcontextwindow}, an inference-time NTK method~\citep{NTK-Aware}, with default hyperparameters. Note that YaRN is provided oracle access to the final sequence length, when that length exceeds the default 4k. Across several benchmarks, our method is comparable to or outperforms all baselines, while generalizing gracefully out of the box, requiring no model adjustment or oracle access to sequence length.
\begin{table}[th]
  \centering
  \caption{Average RULER benchmark accuracy scores by sequence length. Highest average score for each sequence length and model size is in \textbf{bold}; runner-up is \underline{underlined}.}
  \resizebox{0.6\textwidth}{!}{
  \begin{tabular}{l|ccc|ccc}
  \toprule
  & \multicolumn{3}{c|}{\textbf{Llama-1B}} & \multicolumn{3}{c}{\textbf{Llama-3B}} \\
  \cmidrule(lr){2-4}\cmidrule(lr){5-6}
  \textbf{Method}  & 4k & 8k & 16k & 4k & 8k & 16k  \\
  \midrule
   RoPE &  39.72& 0.01& 0.03  & \underline{46.19} &  0.14  &  0.01 \\
   High frequency &  16.04& 7.60& 2.37  & 28.02 &  14.31 &  5.14\\
   HalfRoPE & \textbf{43.07}& 0.14& 0 & \bf 51.28 & 0.4 & 0.03 \\
   YaRN & \underline{40.24}& \underline{35.55}& \underline{30.25} & 43.90 & \textbf{45.09} & \underline{40.14} \\
   \midrule
   \name&{39.15}&29.71 & 14.29 &44.86&37.51&21.95\\
   \name(scaling)&  {39.15}& \textbf{35.64}& \textbf{30.83}  & 44.86 &  \underline{43.39}  &  \bf 42.0\\
  \bottomrule
  \end{tabular}}
    \label{tab:ruler_llama}
\end{table}

  
  \textbf{RULER}~\citep{hsieh2024rulerwhatsrealcontext} measures long context performance on synthetic tasks, such as needle-in-a-haystack retrieval and word counting. Table~\ref{tab:ruler_llama} shows that baselines perform as expected: RoPE and HalfRoPE perform well at 4k training length
    (with HalfRoPE even delivering a boost from its improved semantic encoding), but immediately fall to near-zero performance beyond that. High frequency RoPE degrades less in comparison, but also starts from a much lower score at 4k, as RoPE is scrambling latent information aggressively. YaRN delivers robust extrapolation to lengths 8k and 16k, without major penalty to baseline performance at 4k. Meanwhile, \name with temperature scaling is comparable to YaRN with slight gains for longer sequences.
    It is marginally the strongest evaluated approach overall, but clearly surpasses \name without temperature scaling, so we omit the non-scaled version from further analysis.
    A full breakdown of scores by task is provided in \ref{app:ruler_complete}.


  

\textbf{LongBench}~\citep{bai2024longbenchbilingualmultitaskbenchmark} corroborates our RULER results, with reported averages over five task categories: 
  single document question-answering, multi-document question-answering, few-shot learning, code completion and summarization. 
  We exclude non-English tasks as our models are trained on English data. A full breakdown of scores is provided in \ref{app:longbench_complete}.
  Results in Table~\ref{tab:longbench_llama} mirror Table~\ref{tab:ruler_llama}: RoPE and HalfRoPE drop immediately after 4k, with HalfRoPE delivering a small boost within 4k. High frequency performs stably but poorly, while YaRN is able to bring up performance for long inputs. Our method trails YaRN at 3B scale, but is superior at 1B. We conclude that \name successfully generalizes to longer inputs out of the box. 

    \begin{table}[th]
  \centering
  \caption{LongBench accuracy scores, averaged over 14 English tasks, by sequence length. Highest average score for each sequence length and model size is in \textbf{bold}; runner-up is \underline{underlined}.}
  \resizebox{0.6\textwidth}{!}{
  \begin{tabular}{l|ccc|ccc}
  \toprule
  & \multicolumn{3}{c|}{\textbf{Llama-1B}} & \multicolumn{3}{c}{\textbf{Llama-3B}} \\
  \cmidrule(lr){2-4}\cmidrule(lr){5-7}
  \textbf{Method}  & 4k & 8k & 16k  & 4k & 8k & 16k  \\
  \midrule
   RoPE &  14.61 & 8.23 & 8.73  & \underline{18.62} &  11.36  &  10.42 \\
   High frequency & 11.8 & 11.44 & 11.04 & 14.19 & 13.82 & 13.78 \\
   HalfRoPE & \underline{15.38} & 8.73 & 8.86 & \bf 19.42 & 10.7 & 10.62 \\
   YaRN & 14.84& \underline{14.54}& \underline{14.09}  & 15.87 & \textbf{19.29} & \textbf{19.63}  \\
   \name (scaling) &  \textbf{15.83}& \textbf{15.83} & \textbf{15.80}  & 15.92 &  \underline{17.13}  &  \underline{17.94}\\
  \bottomrule
  \end{tabular}}
    \label{tab:longbench_llama}
\end{table}


\textbf{Commonsense Reasoning} tasks act as a sanity check in 
    Table~\ref{tab:commonsense_reasoning}. 
    Model scores are all similar: RoPE frequency and number of channels have little impact on expressivity within the training length. To the degree that scores differ, \name is in the top-3 for all tasks and settings.
\begin{table}[th]
  \centering
  \caption{Standard evaluation (accuracy) on common sense reasoning tasks}
  \resizebox{0.8\textwidth}{!}{
  \begin{tabular}{l|ccc|c|ccc|c}
  \toprule
  & \multicolumn{4}{c|}{\textbf{Llama-1B}} & \multicolumn{4}{c}{\textbf{Llama-3B}} \\
  \cmidrule(lr){2-5}\cmidrule(lr){6-9}
  \textbf{Method}  & ARC-C & HellaSwag & PIQA & \textbf{Avg.} & ARC-C & HellaSwag & PIQA & \textbf{Avg.} \\
  \midrule
   RoPE &  25.77 &  \bf 44.00 &  \bf 69.26 & 46.34 & 29.18 &  \bf 53.95  &  \bf 72.74 & 51.96 \\
   High frequency &  25.17 &  42.99 &  \bf 69.26  & 45.81 & 29.61 &  53.32 &  71.87 & 51.60 \\
   HalfRoPE & \bf 26.45  & \bf 44.00  & 68.77 & \bf 46.41 & \bf 32.17 & 53.87 & 72.03 & \bf 52.69 \\
   YaRN & 25.60 & 41.89 & 68.61 & 45.37 & 29.10 & 52.46 & 72.20 & 51.25 \\
   \name&  25.60 &  43.58 &  68.88 & 46.02 & 30.03 &  \bf 53.95  &  72.25 & 52.08 \\
  \bottomrule
  \end{tabular}}
  \label{tab:commonsense_reasoning}
\end{table}

\textbf{Extended Context Length} in above evaluations is limited to four times the training length, as tuning-free extension techniques generally cannot avoid performance degradation beyond this point. To evaluate RoPE-ID at greater context lengths, we fine-tune our 1b checkpoints to 128k sequence length, applying YaRN scaling to the vanilla 1b model. Fig.~\ref{fig:finetuning} shows that RoPE-ID and YaRN trade off RULER performance over different context lengths, with RoPE-ID offering improvement on longer inputs. Surprisingly, applying YaRN-style scaling to the RoPE-ID checkpoint results in a tuned model outperforming both techniques at all input lengths. Further discussion and implementation details are in \ref{app:finetuning}, and we leave exploration of possible synergies to future work.
\begin{figure}[h!]
    \centering
    \includegraphics[width=.5\textwidth]{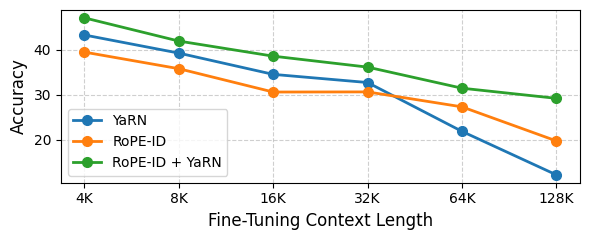}
    \vspace{-12pt}
    \caption{
    Average RULER accuracy scores for fine-tuned 1b models. RoPE-ID and YaRN trade off performance over length, but demonstrate surprising synergy when applied jointly.
    }
    \label{fig:finetuning}
\end{figure}

\textbf{Reapplying Analysis}
    from \sref{sec:clusters} confirms that our trained models obey our geometric framework. We recreate Fig.~\ref{fig:angles} and \ref{fig:pca} for baseline and \name models in \ref{app:analysis}. 
    Our 1B example decoder mirrors the behavior of established LLMs, while \name behaves as expected, delivering stable, and stably separated, key and query clusters over time.

\section{Discussion}
    From empirical analysis we develop a unified understanding of attention geometry and long-context failure modes. Keys and queries form tight clusters in opposing directions, allowing sink tokens to absorb attention weight by default via small $\ell_2$ norm. RoPE inhibits this behavior by merging and dispersing the point clouds, particularly beyond the training length. Overlapped point clouds inflate key/query alignment, preventing the sink token from functioning.
    From this understanding, we derive two natural criteria for length generalization with RoPE: lower-bounded cluster separation in the limit, and attainment of that limit during training.
    We fulfill these criteria and produce stable model behavior by applying RoPE with high frequency to a fraction of channels, and demonstrate strong long-context performance on downstream tasks out of the box. 

    Beyond \name, other approaches based on this analysis are possible (e.g., applying high frequency RoPE to a fraction of \textit{heads}, or manually injected sink tokens as in Hymba~\citep{dong2025hymba}). 
    Additionally, our proposed criteria are necessary for generalization, but not sufficient. \citep{wang2024resonance}, for example, demonstrate that high frequencies can still drift slowly out of distribution radially, a failure mode not covered by \name. Combining such techniques, as in \name itself, could produce even more robust models.
    It is also possible to combine \name with inference-time model adjustments. We leave this, as well as further exploration of long-context fine-tuning of \name models, to future work.

    \textbf{Reproducibility:} Code for \name is provided in \ref{app:pseudocode} with training details in \ref{app:details}. Empirical analysis techniques are straightforward, and evaluation uses standard benchmarks.


\section*{Acknowledgments}
We are grateful to Selcuk Gurses (UAlbany) for valuable discussions on the theoretical analysis. This work was in part supported by NSF grant DMS-2208126, the SUNY-IBM AI Research Alliance grant, the CEAIS grant, and the IBM PhD Fellowship Award. 

\bibliography{iclr2026_conference}
\bibliographystyle{iclr2026_conference}

\newpage
\appendix
\section{Appendix}
\subsection{Theoretical Analysis of Cluster Dispersion under RoPE}\label{app:theory}

Let us first fix the notations: $\r(\X)$ denotes the application of RoPE to a key/query matrix $\X = [\x_1, \dots, \x_n]^\top \in\R^{n\times d}$. Specifically, $\r(\X) := [\Rm_1 \x_1, \dots, \Rm_n \x_n]^\top\in\R^{n\times d}$, where $\Rm_j = \mathrm{Diag}(\Rm_{j,\theta_1},\dots, \Rm_{j,\theta_{d/2}})\in\R^{d\times d}$ with
$\Rm_{j,\theta_k} = \begin{bmatrix}
        \cos(j\theta_k) & -\sin(j\theta_k)\\
        \sin(j\theta_k) & \cos(j\theta_k)
    \end{bmatrix}\in\R^{2\times 2}$ and $\theta_k = \theta^{-2(k-1)/d}$ for all $1\leq k\leq d/2$. 
    For any $\x\in\R^d$, $\x^{(k)}:= [x_{2k-1}, x_{2k}]^\top\in\R^2$ denotes the subvector in the $k^{\mathrm{th}}$ rotation plane/channel.
    $\|\cdot\|_2$ denotes the Euclidean or $\ell_2$ norm of a vector or denotes the spectral norm of a matrix, and $\|\cdot\|_F$ denotes the Frobenius matrix norm. The stable rank of a matrix $\X$ is defined as $\mathrm{srank}(\X) := \frac{\|\X\|_F^2}{\|\X\|_2^2}$. 

\bigskip

\begin{customlemma}{1}[formal]
Suppose $\X = \uu\vv^\top \in\R^{n\times d}$, where $\uu\in\R^n$, $\vv\in\R^d$ and $\|\vv\|_2 = 1$. If $\forall j$, $u_j = \Theta(1)$, and the sequence $\{u_j^2\}_{j=1}^n$ has sublinear growth of total variation in the sequence length $n$, i.e., $\sum_{j=1}^{n-1}|u_{j+1}^2-u_{j}^2| = o(n)$, then as $n$ increases, we have
$$
\frac{\|\r(\X)\|_2}{\|\X\|_2} = \frac{1}{\sqrt{2}}\max_{1\leq k\leq d/2} \alpha_k  + o(1)
$$
where $\alpha_k:=\|\vk\|_2 = \sqrt{v_{2k-1}^2+v_{2k}^2}$ satisfying $\max_{k} \alpha_k\in[\sqrt{2/d},1]$.
\end{customlemma}

\begin{remark}
    The assumption that $\sum_{j=1}^{n-1}|u_{j+1}^2-u_{j}^2| = o(n)$ imposes that the sequence $\{u_j^2\}$ must exhibit a certain degree of monotonicity. Indeed, if $\{u_j^2\}$ is strictly monotonic, then $\sum_{j=1}^{n-1}|u_{j+1}^2-u_{j}^2| = \Theta(1)$. In contrast, if $\{u_j^2\}$ is highly oscillatory, then $\sum_{j=1}^{n-1}|u_{j+1}^2-u_{j}^2| = \Theta(n)$, which violates the assumption.
\end{remark}

\begin{proof}[Proof of Lemma \ref{lem:spectral}]
    The pre-RoPE spectral norm
    $$
    \|\X\|_2 = \sqrt{\lambda_{\max}(\X^\top \X)} = \|\uu\|_2\sqrt{\lambda_{\max}(\vv\vv^\top)} = \|\uu\|_2\|\vv\|_2 = \|\uu\|_2 = \Theta(\sqrt{n}),
    $$
    where $\lambda_{\max}(\cdot)$ denotes the largest eigenvalue. 

    In what follows, we estimate the growth of post-RoPE spectral norm, i.e., $\sqrt{\lambda_{\max}(\r(\X)^\top \r(\X))}$. 
    Since $\|\vk\|_2 = \alpha_k$, due to rational invariance, without loss of generality, we assume $\vk = [\alpha_k, 0]^\top$ for simplicity. 
    
    Consider the Gram matrix $\G := \r(\X)^\top \r(\X)\in\R^{d\times d}$, since $\r(\X) = [u_1  \Rm_1 \vv, \cdots, u_n \Rm_n \vv]^\top\in\R^{n \times d}$, we have
\begin{align*}
    \G & = \r(\X)^\top \r(\X) = \begin{bmatrix}
    u_1 \Rm_1\vv, &
    \cdots, &
    u_n \Rm_n\vv
\end{bmatrix}
\begin{bmatrix}
    u_1  (\Rm_1 \vv)^\top \\
    \vdots \\
    u_n (\Rm_n \vv)^\top
\end{bmatrix} 
= \sum_{j=1}^n u_j^2 (\Rm_j \vv)(\Rm_j \vv)^\top \\
& = \sum_{j=1}^n u_j^2 
\begin{bmatrix}
    \Rm_{j,\theta_1} \vv^{(1)} \\
    \vdots \\
    \Rm_{j,\theta_{d/2}} \vv^{(d/2)}
\end{bmatrix}
\begin{bmatrix}
    (\Rm_{j,\theta_{1}} \vv^{(1)})^\top, &
    \cdots, &
    (\Rm_{j,\theta_{d/2}} \vv^{(d/2)})^\top
\end{bmatrix},
\end{align*}

    {\bf Diagonal Blocks of $\G$}. For $1\leq k\leq d/2$, the $k^{\mathrm{th}}$ diagonal block of $\G$ is given by
    $$
    \G_{k,k} = \sum_{j=1}^n u_j^2(\Rm_{j,\theta_k} \vk)(\Rm_{j,\theta_k} \vk)^\top \in\R^{2\times 2}.
    $$
    Recall that
    $$
    \Rm_{j,\theta_k} = \begin{bmatrix}
        \cos(j\theta_k) & -\sin(j\theta_k)\\
        \sin(j\theta_k) & \cos(j\theta_k)
    \end{bmatrix},
    $$
    we have $\Rm_{j,\theta_k} \vk = \alpha_k[\cos(j\theta_k), \sin(j\theta_k)]^\top$, and thus
    \begin{align*}
         \G_{k,k} & = 
    \alpha_k^2 
    \sum_{j=1}^n u_j^2 \begin{bmatrix}
         \cos^2(j\theta_k) &  \cos(j\theta_k)\sin(j\theta_k) \\
         \cos(j\theta_k)\sin(j\theta_k) &  \sin^2(j\theta_k)
    \end{bmatrix} \\
    & =
    \frac{\alpha_k^2}{2} \sum_{j=1}^n u_j^2
    \begin{bmatrix}
        1 + \cos(2j\theta_k) &  \sin(2j\theta_k) \\
         \sin(2j\theta_k) &  1 -  \cos(2j\theta_k)
    \end{bmatrix} 
    \\
    & = \frac{\alpha_k^2}{2} \|\uu\|_2^2 \, \I_2 + \E_{k,k},
    \end{align*}
    where 
    $$
    \E_{k,k} := \frac{\alpha_k^2}{2} \sum_{j=1}^n u_j^2 \begin{bmatrix}
          \cos(2j\theta_k) & \sin(2j\theta_k) \\
          \sin(2j\theta_k) &  -  \cos(2j\theta_k)
    \end{bmatrix}. 
    $$
  {\bf Off-Diagonal Blocks of $\G$}. Similarly, for $k\neq l$, the $(k,l)^{\mathrm{th}}$ block is
  \begin{align*}
     \E_{k,l} := & \, \G_{k,l}  = \sum_{j=1}^n u_j^2(\Rm_{j,\theta_k} \vk)(\Rm_{j,\theta_l} \vv^{(l)})^\top\\
     = & \,
    \alpha_k^2 \sum_{j=1}^n u_j^2 \begin{bmatrix}
         \cos(j\theta_k)\cos(j\theta_l), &  \cos(j\theta_k)\sin(j\theta_l) \\
         \sin(j\theta_k)\cos(j\theta_l), &  \sin(j\theta_k)\sin(j\theta_l)
    \end{bmatrix} \\
    = & \,
    \frac{\alpha_k^2}{2} \sum_{j=1}^n u_j^2 \begin{bmatrix}
         \cos(j(\theta_{k}+\theta_{l}))+\cos(j(\theta_{k}-\theta_{l})), &  \sin(j(\theta_{k}+\theta_{l}))-\sin(j(\theta_{k}-\theta_{l})) \\
         \sin(j(\theta_{k}+\theta_{l}))+\sin(j(\theta_{k}-\theta_{l})), &  -\cos(j(\theta_{k}+\theta_{l}))+\cos(j(\theta_{k}-\theta_{l}))
    \end{bmatrix}
  \end{align*}

   Therefore, 
   $$\G = \frac{\|\uu\|_2^2}{2}\mathrm{Diag}\left(\alpha_1^2\I_2, \cdots,\alpha_{d/2}^2\I_{2}\right)+ \E.
   $$
   
  {\bf Bounding $\lambda_{\max}(\G)$}. We want to show that $\E$ is a subleading term and is entry-wise $o(n)$. Note that $\alpha_k\leq 1$, for any entry of $\E_{k,k}$, $\forall k$, its magnitude is upper bounded by
  $$
  \sqrt{\left(\sum_{j=1}^n u_j^2 \cos(2j\theta_k)\right)^2+\left(\sum_{j=1}^n u_j^2 \sin(2j\theta_k)\right)^2}= \left|\sum_{j=1}^n u_j^2 e^{2j\theta_k i} \right|.
  $$
   Denoting $S_j := \sum_{t=1}^j e^{2t\theta_k i}$ and using Abel's summation formula, we have
\begin{align*}
    \left|\sum_{j=1}^n u_j^2 e^{2j\theta_k i}\right| & = \left|u_n^2 S_n - \sum_{j=1}^{n-1} (u_{j+1}^2 - u_j^2)S_j\right| \\
    & \leq u_n^2 \left|S_n\right| + \max_{1\leq j\leq n-1} \left|S_j\right| \sum_{j=1}^{n-1} |u_{j+1}^2 - u_j^2|,
\end{align*}
where $|S_j| = |e^{2\theta_k i}|  \left|\frac{1 -  e^{2j \theta_k i}}{1-  e^{2\theta_k i}}\right| \leq \frac{2}{|1-  e^{2\theta_k i}|} = \frac{1}{|\sin(\theta_k)|}= O(1)$ for all $j$, since $\theta_k = \theta^{-2(k-1)/d}\in[\frac{1}{\theta},1]$ is an irrational multiple of $\pi$. Moreover, we have assumed $\sum_{j=1}^{n-1} |u_{j+1}^2 - u_j^2|=o(n)$, so
$\left|\sum_{j=1}^n u_j^2 e^{2j\theta_k i}\right| =  o(n)$.

Similarly, for any entry of the off-diagonal block $\E_{k,l}$, we can also show that its magnitude is upper bounded by
$$
\left|\sum_{j=1}^n u_j^2 e^{j(\theta_k +\theta_l)i} \right| + \left|\sum_{j=1}^n u_j^2 e^{j(\theta_k -\theta_l)i} \right| = o(n),
$$
 since $(\theta_k \pm\theta_l)/2$ are still irrational multiples of $\pi$. 
 
By Weyl's inequality, we have
 $$
\frac{\|\uu\|_2^2}{2}\max_{k}\alpha_k^2 - \|\E\|_2 \leq \lambda_{\max}(\G) \leq \frac{\|\uu\|_2^2}{2}\max_{k}\alpha_k^2 + \|\E\|_2.
 $$
Note that since $d$ is fixed, $\|\E\|_2\leq\|\E\|_F = o(d n) = o(n)$, and thus
$$
\lambda_{\max}(\G) = \frac{\|\uu\|_2^2}{2}\max_{k}\alpha_k^2 + o(n).
$$
Finally, since $\|\uu\|_2 = \Theta(\sqrt{n})$, we have
$$
\frac{\|\r(\X)\|_2}{\|\X\|_2} = \frac{\sqrt{\lambda_{\max}(\G)}}{\|\uu\|_2} = \frac{1}{\sqrt{2}}\max_{1\leq k\leq d/2} \alpha_k + o(1).
$$
\end{proof}
    
\bigskip

\begin{customlemma}{2}
Applying RoPE preserves the Frobenius norm of any $\X$, i.e., $\|\r(\X)\|_F = \|\X\|_F$.
\end{customlemma}

\begin{proof}
[Proof of Lemma \ref{lem:fro}]
Since
$\|\X\|_F^2 = \sum_{j=1}^n \|\x_j\|_2^2$ and 
$\|\r(\X)\|_F^2 = \sum_{j=1}^n \|\Rm_j\x_j\|_2^2$, it suffices to show that for all $j$,
$\|\x_j\|_2^2 = \|\Rm_j\x_j\|_2^2$. Since every diagonal block of $\Rm_j$ is a $2\times 2$ rotation matrix, $\Rm_j$ is also a rotation matrix and thus norm preserving, which completes the proof.
\end{proof}

\bigskip

\begin{customthm}{1}[formal]
        Suppose $\X = \uu\vv^\top \in\R^{n\times d}$, where $\uu\in\R^n$, $\vv\in\R^d$ and $\|\vv\|_2 = 1$. Under the same assumptions on $\uu$ as in Lemma \ref{lem:spectral}, we have
    $$
    \lim_{n\to\infty}\frac{\mathrm{srank}(\r(\X))}{\mathrm{srank}(\X)} = \frac{2}{\max_{1\leq k\leq d/2} \alpha_k^2} \in[2,d],
    $$
    where $\alpha_k:=\|\vk\|_2$, following the definition in Lemma~\ref{lem:spectral}.
\end{customthm}

\begin{proof}[Proof of Theorem \ref{thm:stable_rank}]
Using Lemmas \ref{lem:spectral} and \ref{lem:fro}, we have
$$
\frac{\mathrm{srank}(\r(\X))}{\mathrm{srank}(\X)} = \left(\frac{\|\r(\X)\|_F}{\|\X\|_F}\right)^2 \left(\frac{\|\X\|_2}{\|\r(\X)\|_2}\right)^2 = \frac{2}{\max_{1\leq k\leq d/2} \alpha_k^2  + o(1)}.
$$ 
Since $\sum_{k=1}^{d/2} \alpha_k^2= \|\vv\|_2^2 = 1$, we have $\max_k \alpha_k^2\in[2/d, 1]$. Taking $n\to\infty$ completes the proof.
\end{proof}

\subsection{Attention Map and Sink Token Visualization}
\label{app:masks}
    \begin{figure}[h!]
    \centering
    \includegraphics[width=.8\textwidth]{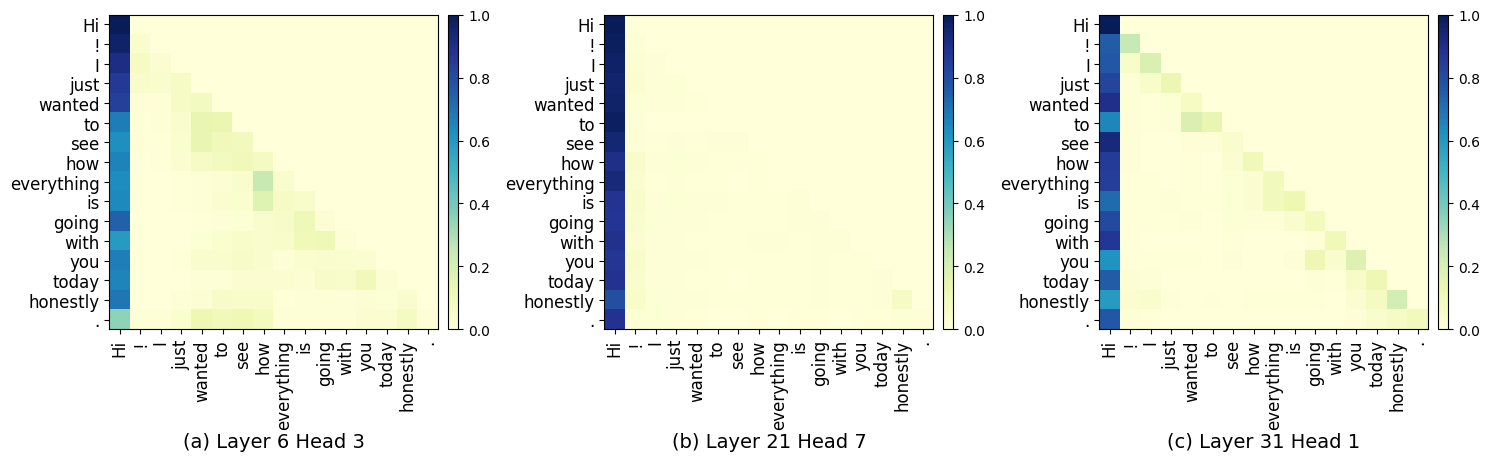}
    \vspace{-10pt}
    \caption{Attention patterns of three heads in LLaMA3-8B. Sink token behavior is clearly observed, even when performing non-trivial token mixing.}
    \label{fig:masks}
    \end{figure}

\subsection{Additional Clustering Analysis}
\label{app:clusteringmetrics}

Here we repeat the analysis performed in Fig.~\ref{fig:angles} for other distance and clustering metrics. Fig.~\ref{fig:dotprods} shows inter- and intra-cluster alignment as measured by dot product rather than cosine distance, better reflecting the actual attention logits. This introduces noise, as embedding norms can shift over time without affecting clustering behavior, and cosine distance is norm-invariant whereas dot products are not. This also introduces large differences between different models. Nevertheless, overall trends are roughly the same. 

\begin{figure}[h!]
    \centering
    \includegraphics[width=.9\textwidth]{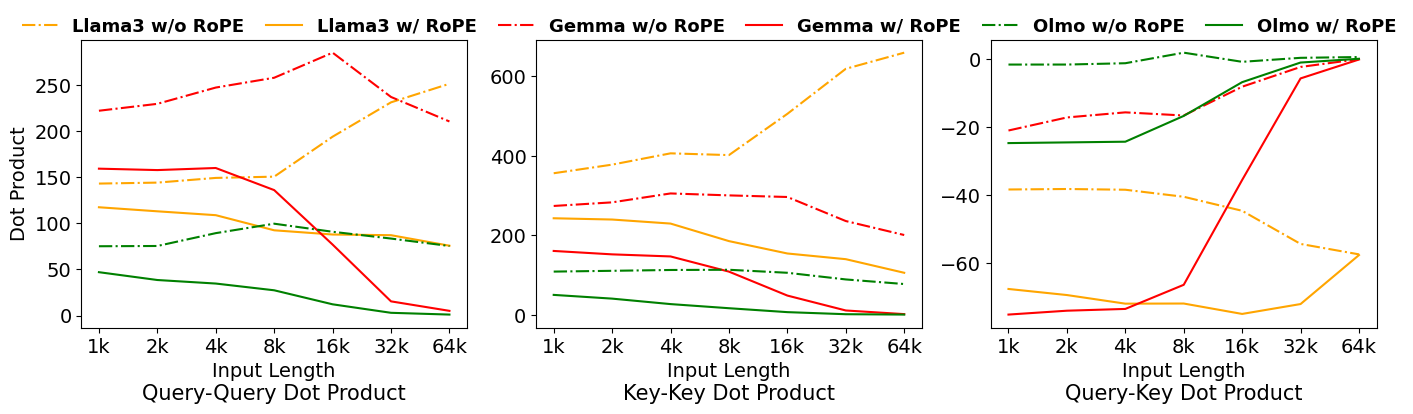}
    \vspace{-10pt}
    \caption{Mean of the pairwise dot product of query and key values across all heads for Llama3-8B, Gemma-7B and OLMo-7B which showcases the effect of RoPE across various context lengths showing similar trend as Fig. \ref{fig:angles}
    }
    \label{fig:dotprods}
\end{figure}

\begin{figure}[h!]
    \centering
    \includegraphics[width=.8\textwidth]{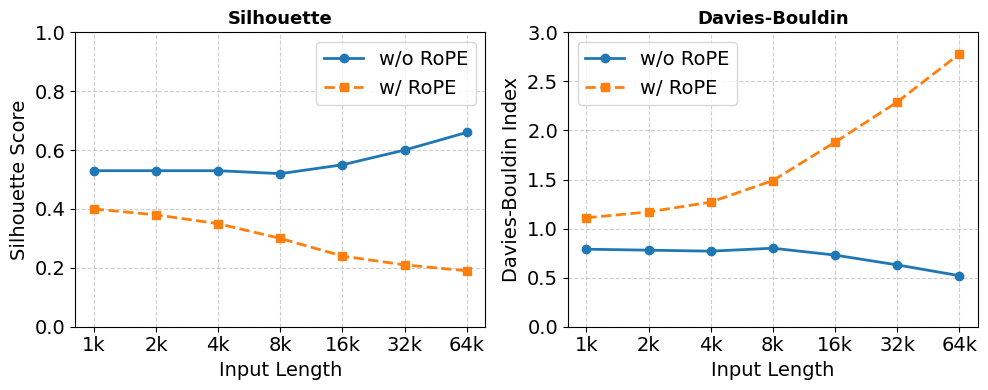}
    \vspace{-10pt}
    \caption{Silhouette Score (left) and Davies-Bouldin Index (right) showcasing the effect of RoPE on Internal Representation Clustering for Llama3-8B-Instruct. Lower Silhouette Score and higher Davies-Bouldin Index represents more overlap.
    }
    \label{fig:sil_dou_metric}
\end{figure}

Fig.~\ref{fig:sil_dou_metric} directly quantifies the degree of clustering, using Silhouette Score (left) and Davies-Bouldin Index (right) in Llama3-8B. Results again mirror Fig.~\ref{fig:angles}: clustering is consistent across sequence lengths prior to RoPE (and even increases beyond the training length), but falls over time once RoPE is applied (Silhouette Score decreases, while Davies-Bouldin Index increases as sequences become long). We conclude that clusters are behaving as described in the main paper.

\subsection{Implementation Details}
\label{app:pseudocode}

  In our approach we apply RoPE to half the channels of each attention head, and adjust the RoPE frequencies to attain desired behaviors. Standard RoPE frequencies interpolate exponentially between 1 and $\frac{1}{\theta}$, where $\theta$ is the base frequency hyperparameter, typically $10k$ or $500k$. We adjust both endpoints of this interpolated scale. First, we must ensure that all frequencies are high enough to complete at least one rotation within our training length of 4k tokens. We set two full rotations as the minimum, as one rotation may not be sufficient to fully eliminate correlation between low frequency channels (i.e. some decay still occurs in the high frequency curve of Fig.~\ref{fig:synthetics} (Left), even after the slowest channel pair finishes a full rotation at the purple line). We therefore update the minimum frequency scale value from $\frac{1}{\theta}$ to $\frac{4\pi}{4096}$. Second, we also pull the maximum frequency scale of 1 toward a more conservative value. Since we apply RoPE to only a fraction of available channels, it is important that the channels be discriminative. RoPE's fastest channel pair completes a cycle in $2\pi \approx 6$ tokens, after which information is effectively lost, as it becomes impossible to disentangle relative position modulo $2\pi$ from content. An effective 6-token window is very aggressive, so we pull back the max frequency to $\frac{2\pi}{32}$, completing a cycle in 32 tokens. 

  While applying high-frequency RoPE to a fraction of channels ensures stable clustering and sink token behavior for long inputs, out-of-distribution behavior can still occur via the softmax activation. Key/query dot products are stable over time by construction, so taking a softmax over an increasing number of IID key/query pairs will increase the softmax denominator, without a corresponding increase to the numerator. The result is the mixture distribution becoming smoother than expected over time. We account for this by introducing temperature scaling based on input length, borrowing from~\citep{peng2023yarnefficientcontextwindow}. The adjustment is $(1+0.1*\text{ln}(\text{min}(4096,n)))^2$, where 4096 is the training length and $n$ is the given input length. 

  Example code for this approach in HuggingFace Transformers is provided below.

\begin{lstlisting}[
    language=Python,
    basicstyle=\ttfamily\small,
    numbers=left,
    caption=Modification of scaling factor within the attention\_interface,
    numberstyle=\tiny\color{gray},
    frame=single,
    showstringspaces=false
]
# src/transformers/models/llama/modeling_llama.py

class LlamaAttention(nn.Module):
    ...
    def forward(...):
        ...
        attn_output, attn_weights = attention_interface(
            self,
            query_states,
            key_states,
            value_states,
            attention_mask,
            dropout=0.0 if not self.training else self.attention_dropout,
            scaling=self.scaling \
            * (0.1 * math.log(max(current_position, 4096) / 4096) + 1)**2,
            **kwargs,
        )
\end{lstlisting}

\begin{lstlisting}[
    language=Python,
    basicstyle=\ttfamily\small,
    numbers=left,
    caption=Modification of modeling\_rope\_utils with our method,
    numberstyle=\tiny\color{gray},
    frame=single,
    showstringspaces=false
]
# src/transformers/modeling_rope_utils.py

ROPE_INIT_FUNCTIONS = {
    ...
    "ourmethod": _compute_our_method_parameters,
}

def _compute_our_method_parameters(
    config, device, seq_len = None, **rope_kwargs
):
    if config is not None and len(rope_kwargs) > 0:
        raise ValueError(...)
    if len(rope_kwargs) > 0:
        base = rope_kwargs["base"]
        dim = rope_kwargs["dim"]
    elif config is not None:
        base = config.rope_theta
        % partial_rotary_factor = config.partial_rotary_factor \
            if hasattr(config, "partial_rotary_factor") else 1.0
        head_dim = getattr(config, "head_dim", None) \
            or config.hidden_size // config.num_attention_heads
        dim = int(head_dim * partial_rotary_factor)

    attention_factor = 1.0  # Unused in this type of RoPE

    logstart = math.log(2 * math.pi / base)  # 1 cycle in ratio steps
    logend = math.log(4 * math.pi / 4096)  # 2 cycles in 4k steps
    pos = torch.arange(0, dim // 2, device=device) / (dim // 2 - 1)
    logfreq = pos * (logend - logstart) + logstart
    inv_freq = logfreq.exp()
    return inv_freq, attention_factor
\end{lstlisting}

\begin{lstlisting}[
    language=Python,
    basicstyle=\ttfamily\small,
    numbers=left,
    caption={Applying RoPE to only half the channels. Due to how HuggingFace implements complex-valued RoPE, the first and third quarter of real-valued channels correspond to the first half of complex-valued channels.},
    numberstyle=\tiny\color{gray},
    frame=single,
    showstringspaces=false
]
# src/transformers/models/llama/modeling_llama.py

def apply_rotary_pos_emb(q, k, cos, sin, position_ids=None, unsqueeze_dim=1):
    cos = cos.unsqueeze(unsqueeze_dim)
    sin = sin.unsqueeze(unsqueeze_dim)

    q_quartile_size = q.shape[-1] // 4
    q1, q2, q3, q4 = torch.split(q, \
        split_size_or_sections=q_quartile_size, dim=-1)
    k_quartile_size = k.shape[-1] // 4
    k1, k2, k3, k4 = torch.split(k, \
        split_size_or_sections=k_quartile_size, dim=-1)

    q_rot = torch.cat((q1, q3), dim=-1)
    k_rot = torch.cat((k1, k3), dim=-1)

    q_rot_embed = (q_rot * cos) + (rotate_half(q_rot) * sin)
    k_rot_embed = (k_rot * cos) + (rotate_half(k_rot) * sin)

    q1_updated, q3_updated = torch.split(q_rot_embed, \
        split_size_or_sections=q_quartile_size, dim=-1)
    k1_updated, k3_updated = torch.split(k_rot_embed, \
        split_size_or_sections=k_quartile_size, dim=-1)

    q_embed = torch.cat((q1_updated, q2, q3_updated, q4), dim=-1)
    k_embed = torch.cat((k1_updated, k2, k3_updated, k4), dim=-1)


    return q_embed, k_embed
\end{lstlisting}

\subsection{Training, Model, and Evaluation Details}
\label{app:details}
\textbf{Evaluation:}
  During evaluation, we take care to avoid inducing out-of-distribution behavior not related to extended context length. In particular, we report point cloud behaviors ``with'' and ``without'' RoPE. Both cases are drawn from the same single forward pass from a given model, with the model unaltered and RoPE applied. Strictly speaking, these point clouds come from \textit{after} and \textit{before} the application of RoPE, respectively.
  This keeps observations
  within-distribution, even when discussing un-RoPEd point clouds inside of a RoPE-using model. Performing the actual attention without RoPE would cascade errors through the model.

\textbf{Training:}
Model training proceeds over 21 billion tokens, with a context length of 4096 and half a million tokens per minibatch. All models are trained across 16 NVIDIA A100s in parallel. We employ a learning rate of $3e-4$, with warmup over 2k steps and cosine decay. Optimizer is AdamW with $\lambda=(.9,.95)$ and weight decay $0.1$. Model architecture follows Llama3, with details provided in Table~\ref{tab:architectures}. 
  \begin{table}[h!]
    \centering
    \caption{Model architectures used for pretraining and evaluation}
    \label{tab:architectures}
    \begin{tabular}{c|cc}\toprule
         \textbf{Parameters}&  \textbf{1B} &  \textbf{3B}\\
         \midrule
         Vocab & 128256 & 128256\\
         Width & 1280 & 2048\\
         Layers & 32 & 48\\
         Heads & 16 & 16\\
         KV heads & 4 & 4\\
         Head dim & 80 & 128\\
         Inner dim & 4096 & 7168\\
         \bottomrule
    \end{tabular}
  \end{table}

\newpage
\subsection{\name HyperParameter Ablations}
\label{app:ablations}

We perform additional ablations on the hyperparameters of \name, namely the high and low frequencies, the number of channels with RoPE applied, and the degree of temperature scaling. We train several additional 1B demo models using the same training procedure, and report average scores from the RULER benchmark. 

We observe a consistent general trend: reducing the number of high-frequency channels improves performance, especially on shorter contexts. However, the model eventually reaches a threshold beyond which length generalization drops sharply. We hypothesize that without enough high-frequency channels, the model instead learns to encode position based on learned patterns of latent drift, which do not generalize to longer contexts. This explains the patterns observed in Table ~\ref{tab:ruler_1b_ablation} and Fig.~\ref{fig:differernt_channel}, where decreasing the number of RoPE channels, and increasing the wavelength of the highest frequency, gradually improves performance, until triggering a catastrophic collapse at longer contexts. We conclude that our hyperparameter choices for channel fraction and highest frequency (50\% of channels, shortest wavelength 32) represent a safe middle-ground. 

Halving the wavelength of the lowest frequency is highly beneficial, showing that one period is indeed not sufficient to decorrelate all rotating channels. This aligns with Fig.~\ref{fig:synthetics}, where the orange curve (RoPE with high enough frequency to complete one period) still performs some FSV decay beyond the training length. Meanwhile, \name, the red curve, holds the FSV ratio constant beyond the training length.

The impact of temperature scaling, shown in Table~\ref{tab:ruler_1b_ablation}, is minimal. Here we raise and lower the exponent of the YaRN temperature scaling formula, and find that the default value of 2 works fine.

\begin{figure}[h!]
    \centering
    \includegraphics[width=1.0\textwidth]{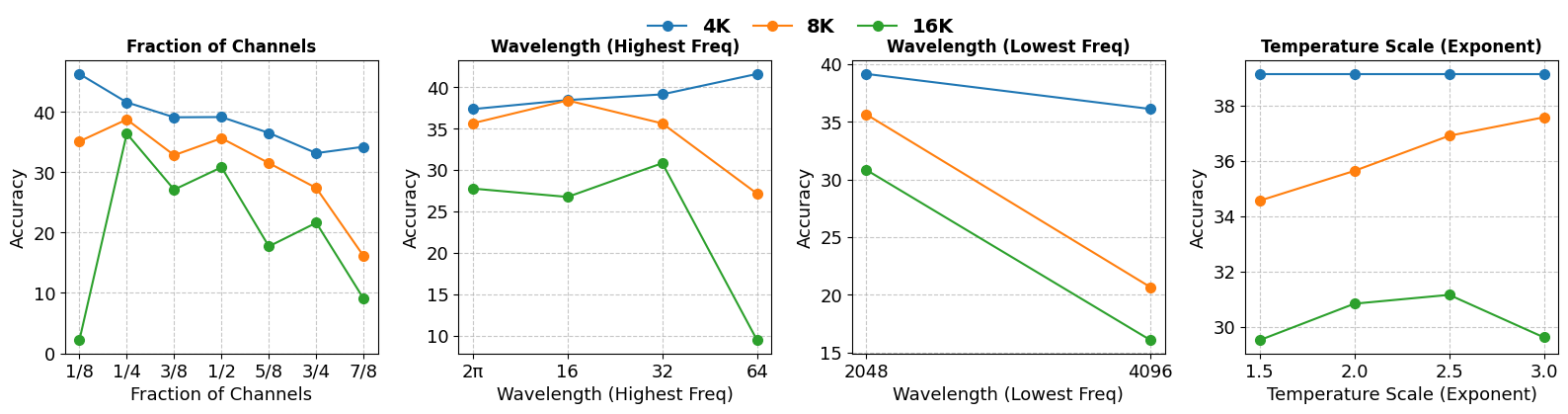}
    \vspace{-10pt}
    \caption{Llama 1B ablation studies with average RULER scores on the y-axis. The x-axis of the 4 plots cover fraction of channels, wavelength (high frequency), wavelength (low frequency) and temperature scaling, respectively. The color coded legend represents context lengths 4k, 8k and 16k.
    }
    \label{fig:differernt_channel}
\end{figure}

These results provide further insight when comparing \name to prior partial-RoPE methods, such as HoPE~\citep{chen2024hopenovelpositionalencoding}. 
While HoPE and \name are algorithmically similar, the differing justifications lead to better-informed design choices for \name. In particular, HoPE corresponds to several suboptimal settings in our ablations: no temperature scaling, a high frequency wavelength of $2\pi$, and a low frequency wavelength of 4k, all of which damage retrieval performance to varying degrees. The number of RoPE channels in HoPE coincidentally matches our setting (50\%) for the 1b model, but climbs to 80\% for 3B, which we find to also be suboptimal. 

\begin{table}[h!]
    \footnotesize
    \centering
    \caption{RULER scores for LlaMA 1B ablations as function of fraction of channels, wavelength, and temperature scaling}
    \begin{tabular}{c|ccc}
         \toprule
         \multicolumn{4}{c}{\textbf{Llama-1B Models}} \\ 
         \midrule
         \textbf{Fraction of}\\\textbf{Channels}&  \textbf{4K} &  \textbf{8K} & \textbf{16K}\\
         \midrule
         1/8 & 46.30 & 35.12 & 2.18\\
         1/4 & 41.59 & 38.75 & 36.41\\
         3/8 & 39.11 & 32.83 & 27.12\\
         \textbf{1/2} & 39.15 & 35.64 & 30.83\\
         5/8 & 36.52 & 31.52 & 17.72\\
         3/4 & 33.17 & 27.40 & 21.63\\
         7/8 & 34.23 & 16.11 & 9.05\\
         \bottomrule
         \toprule
         \textbf{Wavelength}\\\textbf{(Highest Freq)}&  \textbf{4K} &  \textbf{8K} & \textbf{16K}\\
         \midrule
         2$\pi$ & 37.37 & 35.66 & 27.77\\
         16 & 38.46 & 36.52 & 26.76\\
         \textbf{32} & 39.15 & 35.64 & 30.84\\
         64 & 41.62 & 27.16 & 9.46\\
         \bottomrule
         \toprule
         \textbf{Wavelength}\\\textbf{(Lowest Freq)}&  \textbf{4K} &  \textbf{8K} & \textbf{16K}\\
         \midrule
         \textbf{2048} & 39.15 & 35.64 & 30.83\\
         4096 & 36.11 & 20.66 & 16.07\\
         \bottomrule
         \toprule
         \textbf{Temperature Scale}\\\textbf{(Exponent)}&  \textbf{4K} &  \textbf{8K} & \textbf{16K}\\
         \midrule
         1.5 & 39.15 & 34.56 & 29.51\\
         \textbf{2.0} & 39.15 & 35.64 & 30.83\\
         2.5 & 39.15 & 36.91 & 31.15\\
         3.0 & 39.15 & 37.58 & 29.61\\
         \bottomrule
    \end{tabular}
    \label{tab:ruler_1b_ablation}
\end{table}

\subsection{RULER Benchmark}
\label{app:ruler_complete}
        
\begin{table}[H]
    \centering
    \caption{Performance of different methods on the RULER benchmark. Highest average score for each sequence length and model size is in \textbf{bold}; runner-up is \underline{underlined}.}
    \resizebox{\textwidth}{!}{
    \begin{tabular}{lr|rrrrrrrrrrrrr|r}
    \toprule
    \textbf{Method} & Seq. & N-S1 & N-S2 & N-S3& N-MK1 & N-MK2& N-MK3 & N-MV & N-MQ & VT & CWE & FWE & QA-1 & QA-2 & Avg. \\
    \midrule
    \multicolumn{16}{c}{\textbf{1B Models}} \\ 
    \midrule
    \multirow{3}{*}{RoPE} & 4k & 100& 100& 96.6& 73.6& 2.4& 3.4& 23.45& 24.65& 10.24& 21.76& 27.4& 22.67& 10.24& 39.72\\
             & 8k & 0& 0& 0& 0& 0& 0& 0& 0& 0& 0& 0& 0.13& 0& 0.01\\
             & 16k & 0& 0& 0& 0& 0& 0& 0& 0& 0& 0& 0& 0.4 & 0 & 0.03\\
    \midrule
    \multirow{3}{*}{High Frequency}  & 4k & 42& 44.8& 35& 37.4& 0.2& 3.6& 7.05& 4.85& 0& 0.1& 1& 16.95& 15.6
& 16.04\\
      & 8k & 18.8& 15.6& 15& 19.4& 0& 0.4& 6.35& 1.2& 0& 0.04& 2.87& 6.13& 13& 7.60\\
      & 16k & 9& 3& 1& 3& 0& 0& 1.2& 0.1& 0& 0& 0.2& 4.73& 8.6& 2.37\\
    \midrule
    \multirow{3}{*}{HalfRoPE}  & 4k & 100& 100& 96& 76.4& 22& 17.2& 12.2& 10.95& 3.88& 48.72& 27.73& 22.42& 22.4
& \textbf{43.07}\\
      & 8k & 0& 0& 0& 0& 0& 0& 0& 0& 0& 0& 0.07& 1.6& 0.2
& 0.14\\
      & 16k & 0& 0& 0& 0& 0& 0& 0& 0& 0& 0& 0& 0& 0& 0\\
    \midrule
    \multirow{3}{*}{YaRN}  & 4k & 100& 100& 96.6& 73.6& 2.4& 3.4& 23.45& 24.65& 10.24& 21.76& 27.4& 22.67& 17
& \underline{40.24}
\\
      & 8k & 97.6& 92.2& 92.4& 62& 1.6& 1.8& 22.2& 23.1& 21.8& 2.8& 17.07& 11.57& 16
& \underline{35.55}
\\
      & 16k & 99.8& 85.8& 68& 44.2& 1.6& 1.4& 20.35& 17.7& 18.24& 0.34& 12& 11.07& 12.8& \underline{30.25}\\

    \midrule
    \midrule
    \multirow{3}{*}{\name}  & 4k & 100 & 100 & 85.6 & 62 & 3 & 2.6 & 29.7 & 30.85 & 0.52 & 12.12 & 35.4 & 27.52 & 19.6& 39.15
\\
      & 8k & 100 & 100 & 33.4 & 40.4 & 0.2 & 0.6 & 21.8 & 20.55 & 11.44 & 2.44 & 29.33 & 10.28 & 15.8 & 29.71
\\
      & 16k & 96.4 & 19.4 & 1.2 & 14.2 & 0 & 0 & 9.65 & 2.1 & 2.36 & 0 & 20.33 & 8.73 & 11.4 & 14.29\\

    \midrule
    \multirow{3}{*}{\name (Scaling)}  & 4k & 100& 100& 85.6& 62& 3& 2.6& 29.7& 30.85& 0.52& 12.12& 35.4& 27.52& 19.6& {39.15}
\\
      & 8k & 100 & 100 & 70.6 & 50.8 & 1.4 & 1.2 & 26.3 & 34.55 & 12.8 & 6.16 & 33.13 & 11.02 & 15.4 & \textbf{35.64}
\\
      & 16k & 100 & 98.6 & 37 & 40.8 & 1.2 & 0.4 & 29.95 & 23.5 & 6.64 & 0.38 & 32.13 & 15.82 & 14.4 & \textbf{30.83}\\

    \midrule
    \multicolumn{16}{c}{\textbf{3B Models}} \\ 
    \midrule
    \multirow{3}{*}{RoPE} & 4k & 100& 100& 93.4& 56& 7.2& 3.8& 39.55& 53.55& 12.96& 41.28& 36.07& 30.25& 26.4& \underline{46.19}\\
             & 8k & 0& 0& 0& 0& 0& 0& 0& 0& 0& 0& 0.07& 1.8& 0& 0.14\\
             & 16k & 0& 0& 0& 0& 0& 0& 0& 0& 0& 0& 0& 0.13& 0& 0.01\\
    \midrule
    \multirow{3}{*}{High Frequency}& 4k & 59.8& 57.4& 47.8& 26.89& 40.6& 18.4& 18.4& 17.6& 0& 13.9& 13.2& 29.53& 20.8& 28.02\\
      & 8k & 27.8& 24.4& 25.6& 23.8& 1.8& 4.8& 17.45& 15.4& 0& 10.72& 8.2& 9.4& 16.6& 14.31\\
      & 16k & 8.8& 9.2& 5.2& 7.8& 0& 0.4& 7.65& 2.05& 0& 1.58& 5.53& 7.27& 11.4& 5.14\\
    \midrule
    \multirow{3}{*}{HalfRoPE}  & 4k & 100& 100& 99.6& 85& 5.6& 4.6& 58.9& 55.7& 15.8& 44& 38.4& 33.4& 25.6
& \textbf{51.28}
\\
      & 8k & 0& 0& 0& 0& 0& 0& 0& 0& 0& 0.82& 3.2& 0.93& 0.2& 0.40\\
      & 16k & 0& 0& 0& 0& 0& 0& 0& 0& 0& 0& 0& 0.13& 0.2& 0.03\\
    \midrule
    \multirow{3}{*}{YaRN}  & 4k & 100& 98.2& 91.8& 57.2& 9.8& 4.8& 64.3& 38.6& 20.88& 14.66& 21.27& 28& 21.2
& 43.90
\\
      & 8k & 100 & 100 & 98.6 & 71 & 1.8 & 4.4 & 42.25 & 60.25 & 20.32 & 13.94 & 29.8 & 17.18 & 26.6 & \textbf{45.09}
\\
      & 16k & 100 & 99.8 & 92.6 & 63.4 & 7.2 & 1.4 & 34.55 & 47.65 & 11.76 & 8.66 & 18.8 & 15.82 & 20.2 & \underline{40.14}\\

    \midrule
    \midrule
    \multirow{3}{*}{\name}& 4k & 100& 100& 97.8& 77.4& 9.4& 8.4& 28.2& 37.25& 12.84& 13.18& 42.27& 30.88& 25.6
& 44.86
\\
      & 8k & 100& 99& 88.6& 45.4& 0.6& 1.8& 23.7& 36.5& 13.4& 6.54& 36.4& 14.28& 21.4
& 37.51
\\
      & 16k & 73.6& 73.4& 13.2& 32.2& 0& 0.2& 8.25& 8.4& 18.72& 3.72& 26.8& 12.87& 14& 21.95\\

    \midrule
    \multirow{3}{*}{\name (Scaling)}& 4k & 100 & 100 & 97.8 & 77.4 & 9.4 & 8.4 & 28.2 & 37.25 & 12.84 & 13.18 & 42.27 & 30.88 & 25.6
& 44.86
\\
      & 8k & 100 & 100 & 96 & 55.8 & 6.6 & 3.4 & 43.4 & 48.9 & 20.6 & 12.36 & 37 & 15.62 & 24.4
& \underline{43.39}
\\
      & 16k & 90.8 & 99.2 & 86 & 59.6 & 7.2 & 0.4 & 48 & 42.4 & 36.56 & 10.24 & 31.07 & 17.07 & 17.4 
& \textbf{42.0}\\

    \bottomrule
    \end{tabular}

    }
\end{table}
\clearpage

\subsection{LongBench Benchmark}
\label{app:longbench_complete}
    14 English Tasks for 5 Different Categories used for evaluation:
    
    \textbf{Single Document QA}: NarrativeQA, Qasper, and MultiFieldQA-en\\
    \textbf{Multi-Document QA}: 2WikiMultihopQA, HotpotQA, MuSiQue\\
    \textbf{Summarization}: GovReport, MultiNews, QMSum\\
    \textbf{Few-shot Learning}: SAMSum, TREC, TriviaQA\\
    \textbf{Code Completion}: LCC, RepoBench-P\\
    \begin{table}[!ht]
\centering
\caption{Performance of different methods on LongBench, averaged by task type. Highest total average score for each sequence length and model size is in \textbf{bold}; runner-up is \underline{underlined}.}
\label{tab:longbench}
\resizebox{0.9\textwidth}{!}{
\begin{tabular}{lc|ccccc|r}

\toprule

\textbf{Methods} & \bf Seq. & \textbf{Single-Doc QA} & \textbf{Multi-Doc QA} & \textbf{Summarization} & \textbf{Few-Shot Learning} & \textbf{Code} & \bf Avg.\\ 

\midrule
\multicolumn{8}{c}{\textbf{1B Models}} \\ 
\midrule

\multirow{3}{*}{RoPE} & 4k & 6.65 & 4.32 & 14.22 & 27.27 & 23.55 & 14.61\\
 & 8k & 4.53 & 1.69 & 11.48 & 7.55 & 19.75 & 8.23    \\
 & 16k & 4.84 & 2.08 & 12.74 & 8.19 & 19.33 & 8.73   \\

\midrule

\multirow{3}{*}{High Frequency} & 4k & 5.26 & 4.22 & 12.97 & 19.3 & 19.96 & 11.8    \\
 & 8k & 5.31 & 4.22 & 12.61 & 17.61 & 20.47 & 11.44    \\
 & 16k & 5.2 & 3.92 & 12.49 & 16.71 & 19.77 & 11.04   \\

\midrule

\multirow{3}{*}{HalfRoPE} & 4k & 6.94 & 4.72 & 16.07 & 28.73 & 22.99 & \underline{15.38}    \\
 & 8k & 4.96 & 1.94 & 15.89 & 6.78 & 16.77 & 8.73    \\
 & 16k & 5.25 & 2 & 16.17 & 6.5 & 17.17 & 8.86   \\

\midrule

\multirow{3}{*}{YaRN}& 4k & 6.86 & 4.35 & 14.94 & 28.07 & 22.54 & 14.84    \\
 & 8k & 6.83 & 4.77 & 15.29 & 26.76 & 21.33 & \underline{14.54}    \\
 & 16k & 6.78 & 5.18 & 15.26 & 25.59 & 19.42 & \underline{14.09}   \\

\midrule
\midrule

\multirow{3}{*}{\name (scaling)} & 4k & 6.71 & 4.80 & 14.61 & 31.68 & 24.11 & \textbf{15.83}\\
 & 8k & 7.18 & 5.23 & 14.69 & 31.39 & 23.11 & \textbf{15.83}\\
 & 16k & 7.01 & 5.35 & 15.40 & 30.55 & 23.12 & \textbf{15.80}\\  

\midrule
\multicolumn{8}{c}{\textbf{3B Models}} \\ 
\midrule

\multirow{4}{*}{RoPE} & 4k & 7.32 & 4.81 & 14.79 & 35.13 & 37.28 & \underline{18.62}\\
 & 8k & 5.19 & 2.57 & 13 & 11.03 & 31.84 & 11.36    \\
 & 16k & 5.12 & 2.51 & 11.67 & 11.22 & 27.15 & 10.42   \\

\midrule

\multirow{3}{*}{High Frequency} & 4k & 5.42 & 4.52 & 12.61 & 25.85 & 26.69 & 14.19    \\
 & 8k & 5.82 & 4.6 & 12.64 & 23.06 & 27.58 & 13.82    \\
 & 16k & 5.93 & 4.51 & 12.51 & 22.7 & 28.02 & 13.78   \\

\midrule

\multirow{3}{*}{HalfRoPE} & 4k & 7.75 & 4.81 & 17.14 & 40.68 & 30.4 & \textbf{19.42}    \\
 & 8k & 5.13 & 2.23 & 16.81 & 9.84 & 23.89 & 10.7    \\
 & 16k & 4.99 & 2.23 & 16.11 & 10.87 & 23.03 & 10.62   \\

\midrule

\multirow{3}{*}{YaRN}& 4k & 6.5 & 4.96 & 15.2 & 29.74 & 26.50 & 15.87    \\
 & 8k & 7.96 & 5.46 & 15.57 & 40.30 & 31.12 & \textbf{19.29}    \\
 & 16k & 8 & 5.94 & 16.45 & 42 & 28.81 & \textbf{19.63}   \\

\midrule
\midrule

\multirow{3}{*}{\name (scaling)} & 4k & 7.71 & 5.23 & 13.51 & 32.83 & 22.50 & {15.92}    \\
& 8k & 8.52 & 5.65 & 14.55 & 36.37 & 22.24 & \underline{17.13}\\
& 16k & 8.96 & 6.23 & 15.90 & 37.93 & 22.03 & \underline{17.94} \\

\bottomrule
\end{tabular}
    }
\end{table}

\newpage
\subsection{Long-Context Fine-Tuning with \name}
\label{app:finetuning}

While our analysis focuses on \name as a tuning-free approach to length generalization, we can also combine it with fine-tuning to further extend the effective context length. Here we take our trained Llama 1B models and tune them to 128k context length in stages: first, we load the 4k model checkpoint and apply any relevant RoPE frequency scaling. Then, we continue pretraining for 5k steps, with sequence length increased to 32k. The total tokens per batch is held constant at 500k, and learning rate warms up over 250 steps until it reaches the final LR of the previous checkpoint ($3e-5$), where it is held constant. We then repeat the process for another 5k steps, going from 32k sequence length to 128k. 

We extend three models in this fashion: first, we apply YaRN scaling to the baseline RoPE model during each jump in sequence length. Second, we tune the \name model with no adjustment to RoPE frequencies. Third, we tune the \name model, but with YaRN-style scaling also applied. YaRN cannot be applied directly to \name models as the default YaRN hyperparameters do not work for such high frequencies. We therefore set the scaling thresholds to the highest and lowest frequencies and interpolate in-between (the highest frequency is unchanged, and the lowest frequency scales up by $L'/L$, where $L',L$ are the new and old sequence lengths, respectively). 

Results for RULER are given in Table~\ref{tab:ruler_tune_128k} and Fig.~\ref{fig:finetuning}. The tuned \name model exhibits better length generalization at context length 64k and above. Notably, the combination of \name and YaRN-style scaling achieves superior performance at all context lengths, compared to either method alone. This suggests that YaRN and RoPE-ID do not address the same out-of-distribution behaviors, indicating possible synergy between these approaches. We leave such exploration for future work.

\begin{table}[h!]
\footnotesize
\centering
\caption{RULER scores for LlaMA 1B models up to 128k context length, where experiments cover tuning with YaRN, RoPE-ID, and RoPE-ID + YaRN.}
\begin{tabular}{c|ccc}\toprule
\multicolumn{4}{c}{\textbf{Llama-1B Models}} \\ 
\midrule
\textbf{Fine-Tuning}&  \textbf{YaRN} &  \textbf{RoPE-ID} & \textbf{RoPE-ID}\\&&&\textbf{+ YaRN}\\
\midrule
4K & 43.29 & 39.49 & 47.09\\
8K & 39.27 & 35.81 & 41.94\\
16K & 34.55 & 30.63 & 38.59\\
32K & 32.74 & 30.66 & 36.16\\
64K & 21.90 & 27.34 & 31.48\\
128K & 12.25 & 19.78 & 29.23\\
\bottomrule
\end{tabular}
\label{tab:ruler_tune_128k}
\end{table}


\newpage
\subsection{Repeated Clustering Analysis for Trained Models}

We repeat our original analysis on our trained 1B models (baseline and \name). Our baseline model exhibits the same behavior observed in state of the art LLMs, while our \name model exhibits the desired stable behavior and consistent clustering across sequence lengths.

\label{app:analysis}
\begin{figure}[h!]
      \centering
      \includegraphics[width=0.9\textwidth]{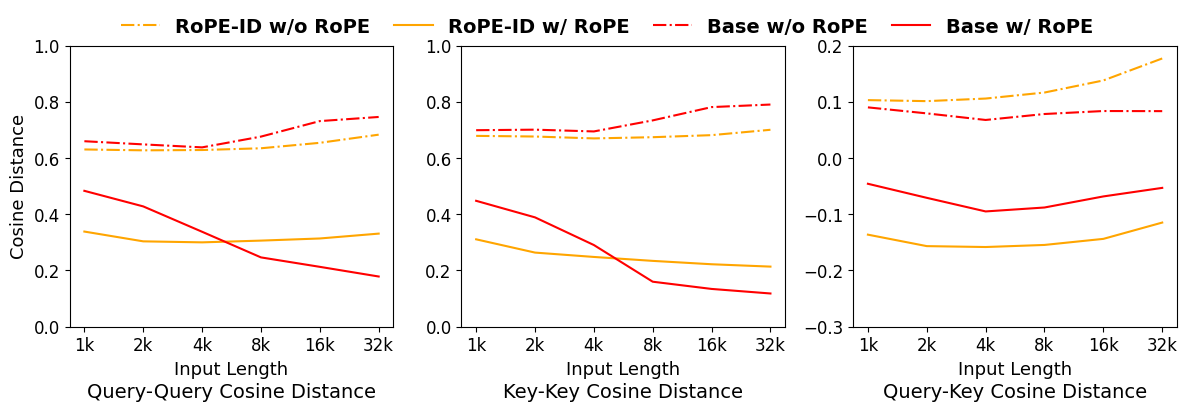}
      \vspace{-10pt}
      \caption{Pairwise angular distances within- and between-clusters for our trained models. Base model matches prior LLM observations, including an inflection point for key-query product with RoPE at the training length (4k). \name maintains stable behavior over time.}
    \label{fig:angles_1b}
\end{figure}
  
    \begin{figure}[h!]
      \centering
      \includegraphics[width=0.9\textwidth]{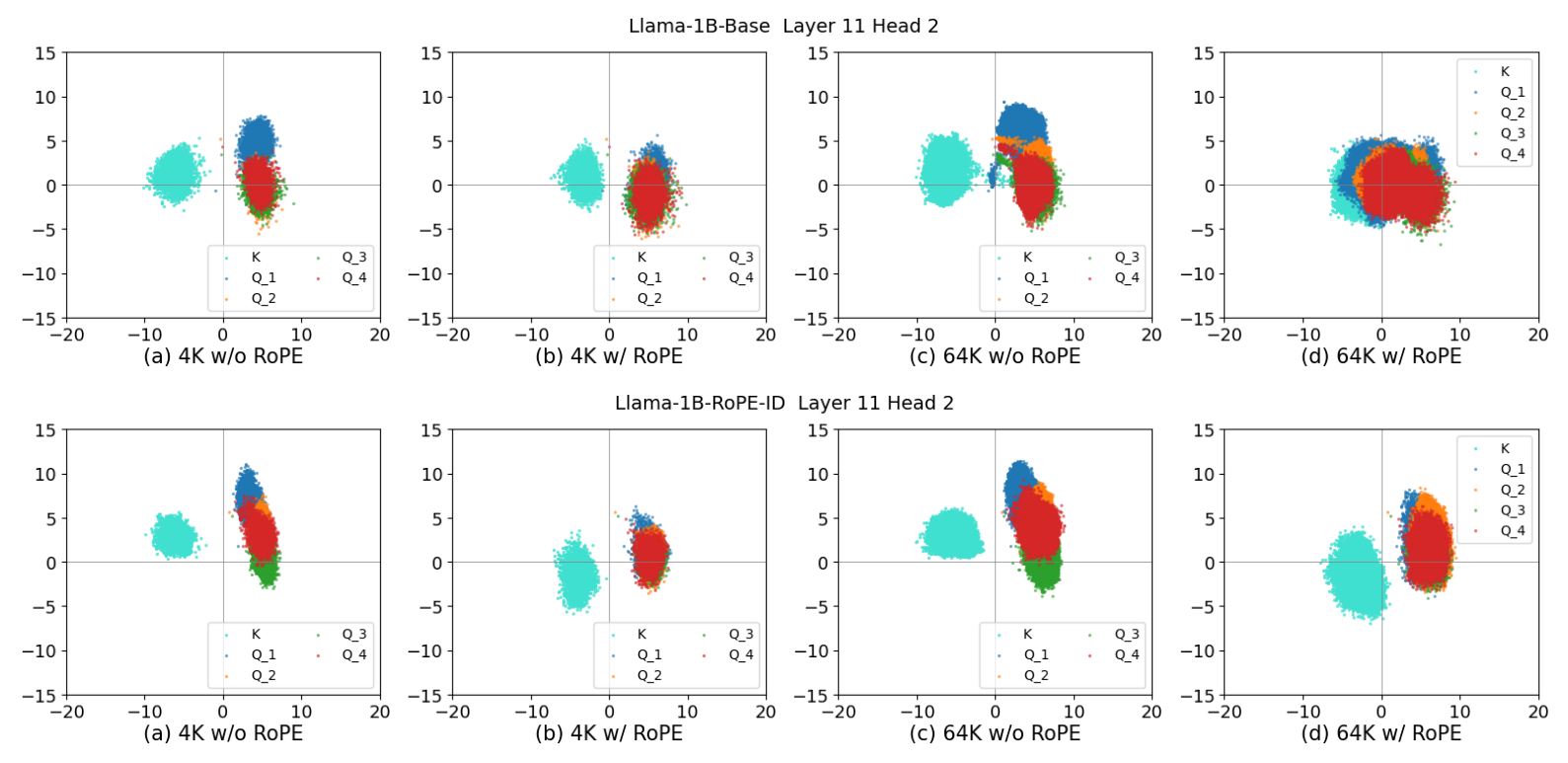}
      \vspace{-10pt}
      \caption{Clustering behavior for our trained models. Vanilla transformer (top) matches prior LLM observations; \name (bottom) maintains stable cluster separation.}
    \label{fig:pca_1b}
    \end{figure}

\newpage
\subsection{Additional Cluster Visualizations}
\label{app:morepca}
We repeat Fig.~\ref{fig:pca} for additional layers and heads. The same general trend can be observed, where separated clusters disperse and overlap when RoPE is applied at longer contexts. We randomly sample 16 Key heads and their corresponding Query heads from  Llama3-8B-Instruct, and 8 Query-Key head pairs from Gemma-7b and Olmo-7b.

\begin{figure}[h!]
    \centering
    \includegraphics[width=1.0\textwidth]{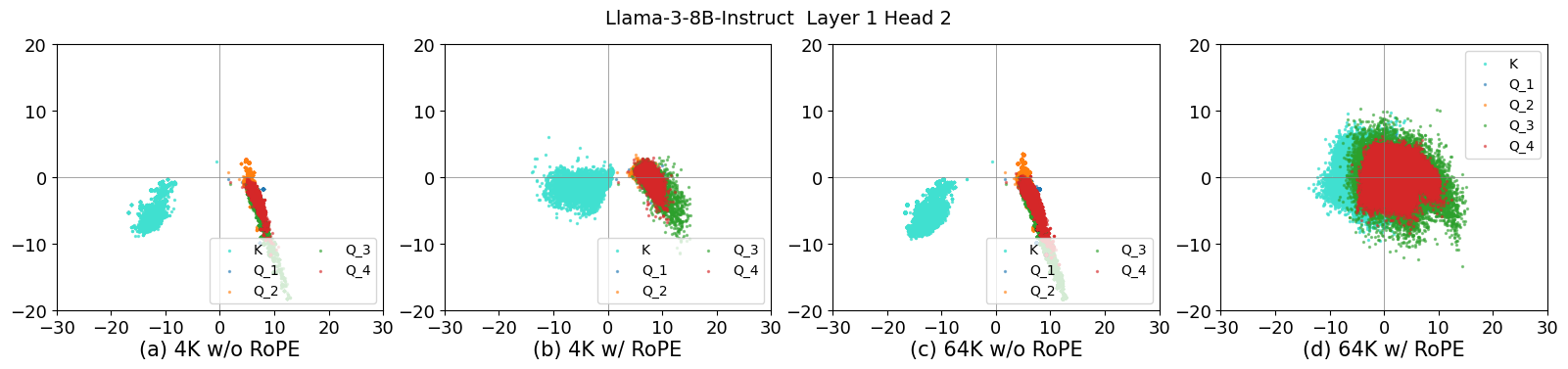}
\end{figure}

\vspace{-20pt}
\begin{figure}[h!]
    \centering
    \includegraphics[width=1.0\textwidth]{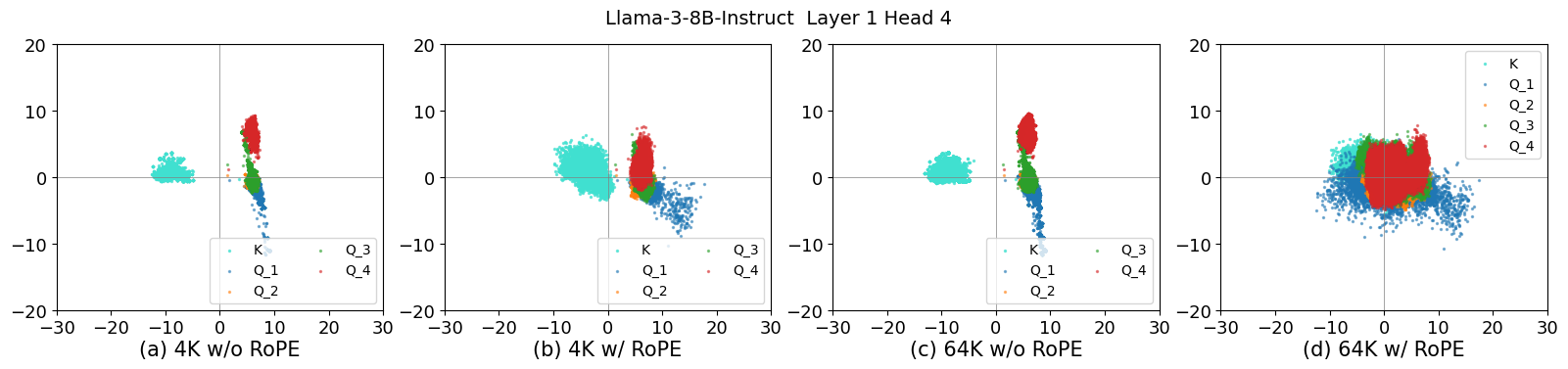}
\end{figure}

\vspace{-20pt}
\begin{figure}[h!]
    \centering
    \includegraphics[width=1.0\textwidth]{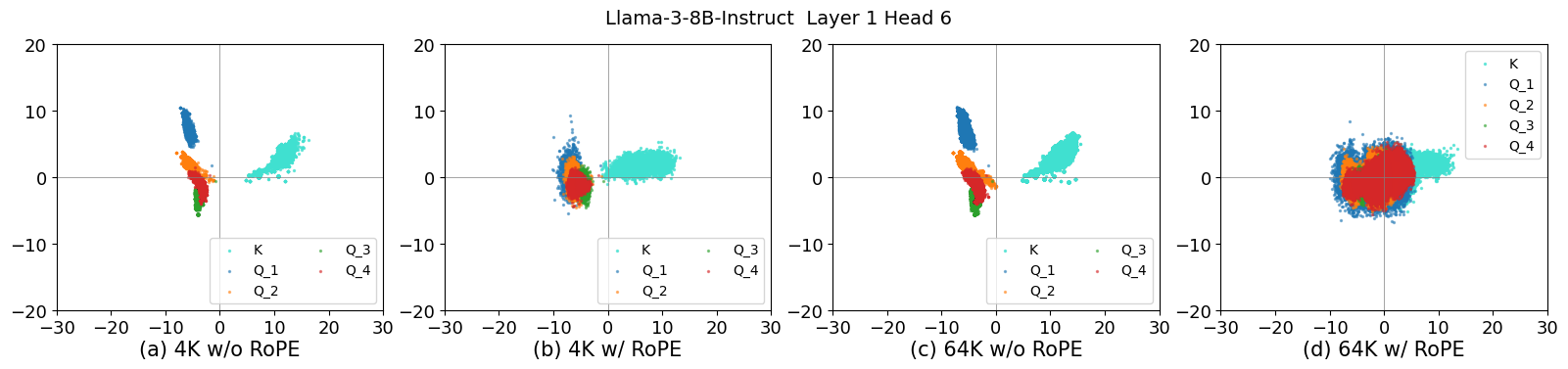}

\end{figure}

\vspace{-20pt}
\begin{figure}[h!]
    \centering
    \includegraphics[width=1.0\textwidth]{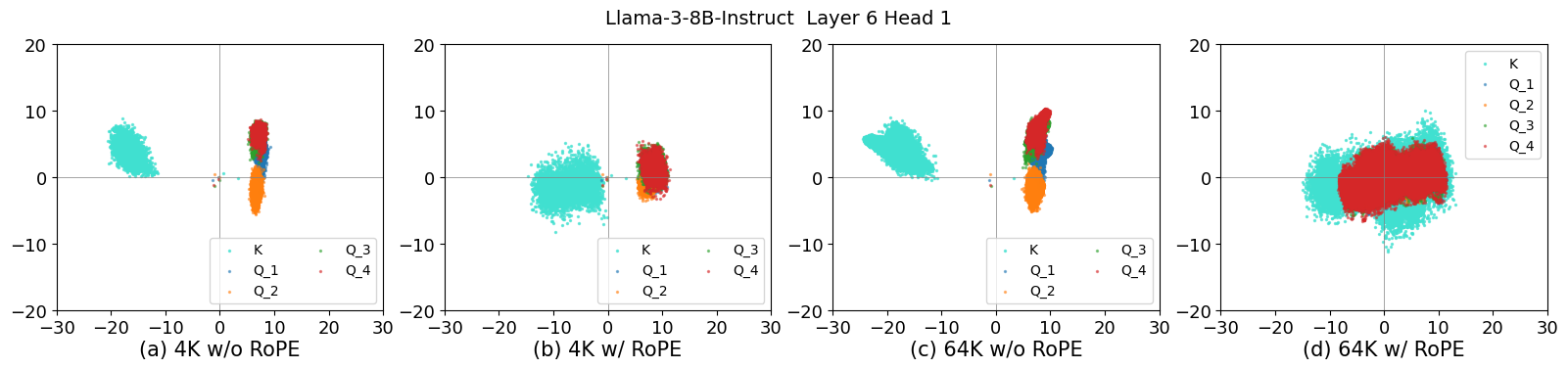}

\end{figure}

\vspace{-20pt}
\begin{figure}[h!]
    \centering
    \includegraphics[width=1.0\textwidth]{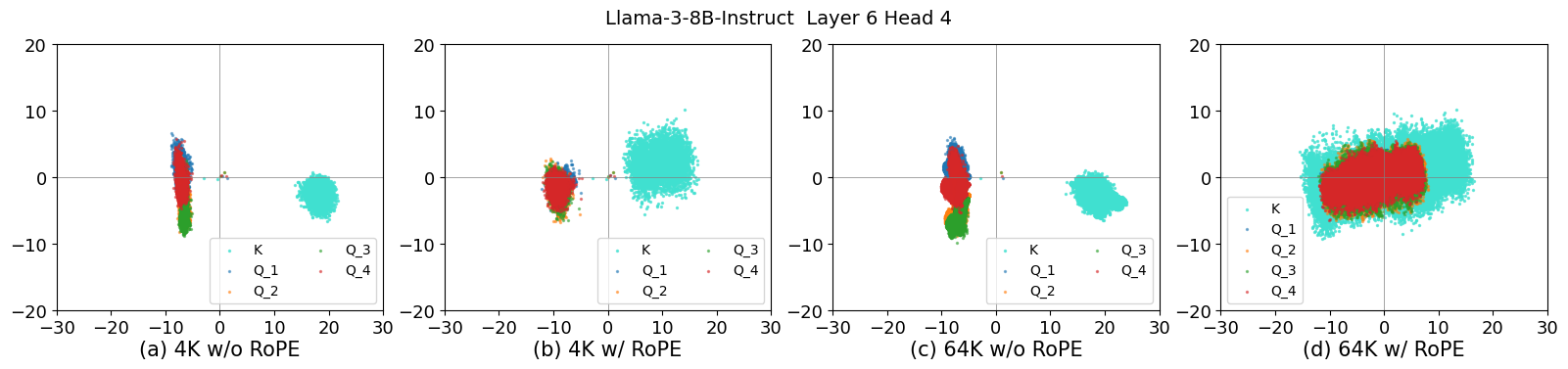}

\end{figure}

\vspace{-20pt}
\begin{figure}[h!]
    \centering
    \includegraphics[width=1.0\textwidth]{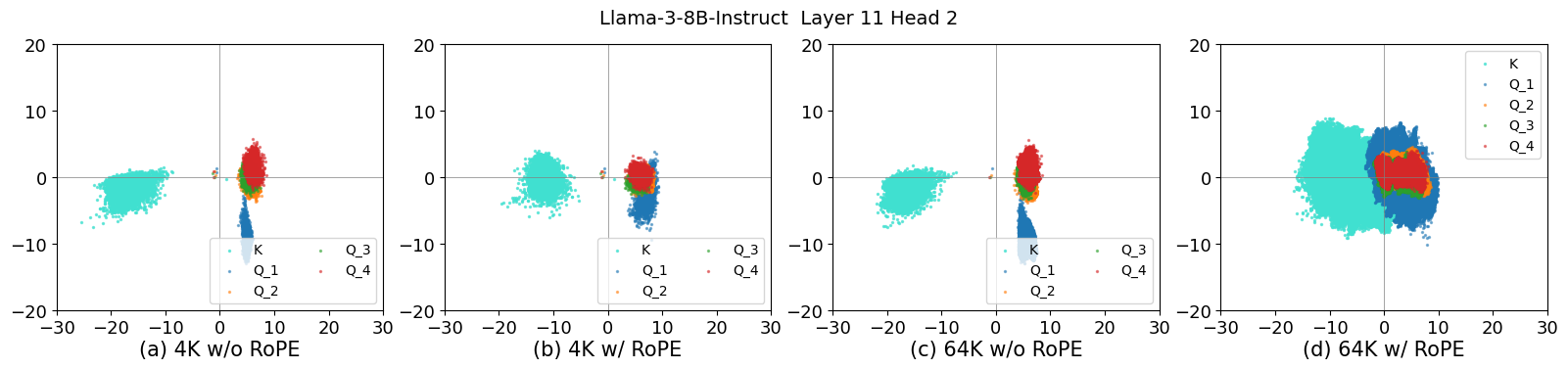}

\end{figure}

\vspace{-30pt}
\begin{figure}[h!]
    \centering
    \includegraphics[width=1.0\textwidth]{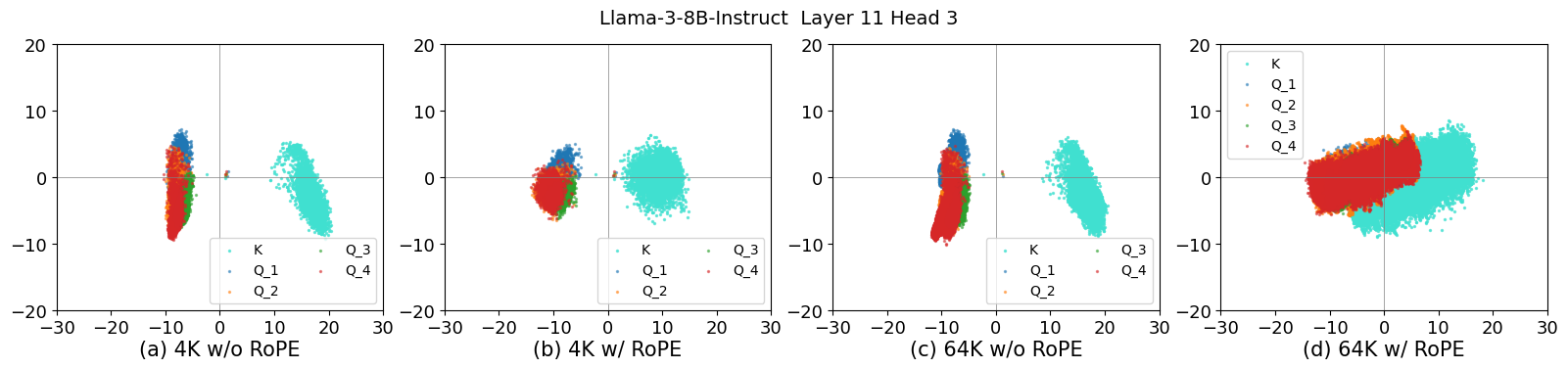}

\end{figure}

\vspace{-30pt}
\begin{figure}[h!]
    \centering
    \includegraphics[width=1.0\textwidth]{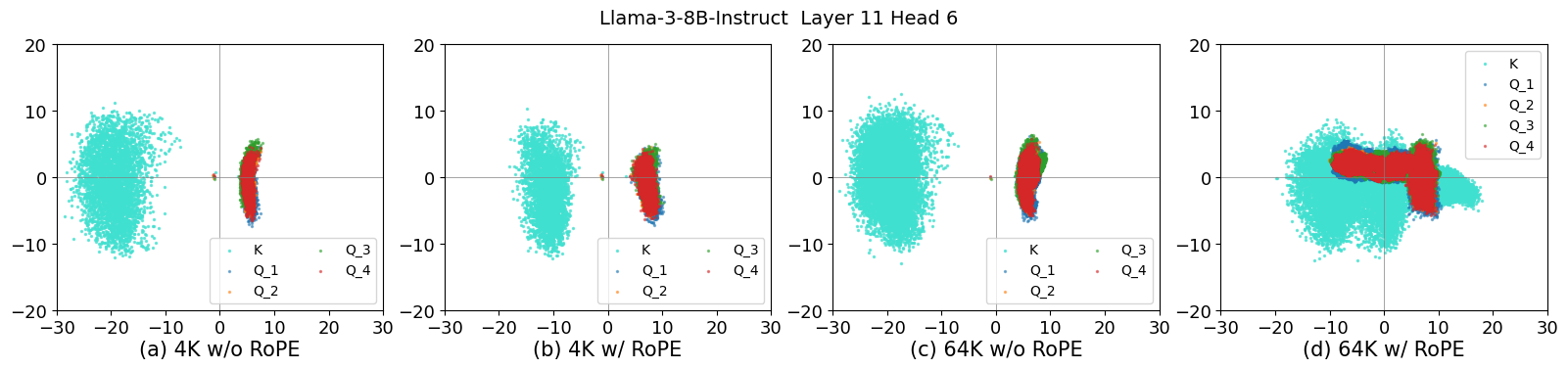}

\end{figure}
\vspace{-30pt}
\begin{figure}[h!]
    \centering
    \includegraphics[width=1.0\textwidth]{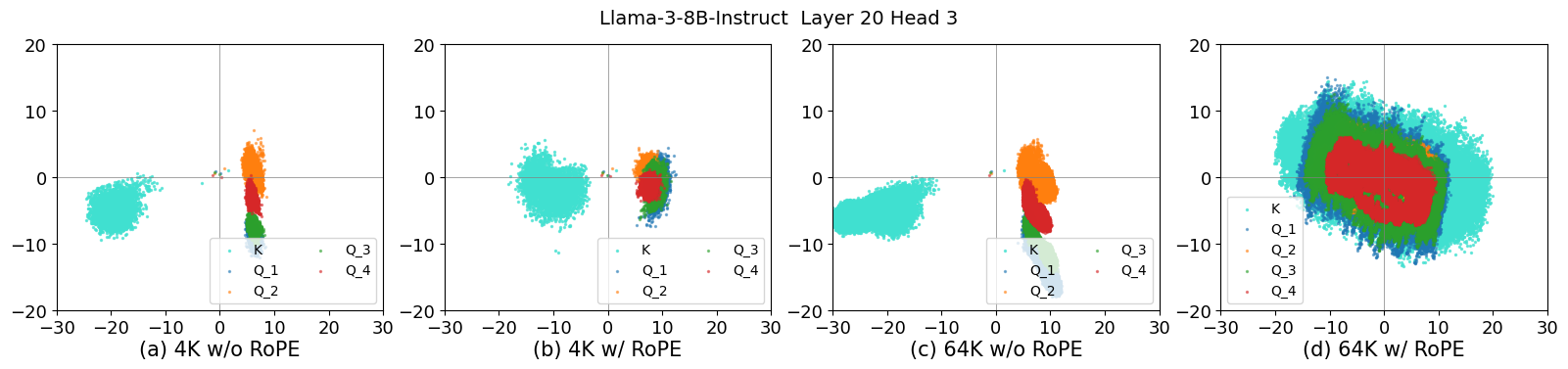}

\end{figure}
\vspace{-30pt}
\begin{figure}[h!]
    \centering
    \includegraphics[width=1.0\textwidth]{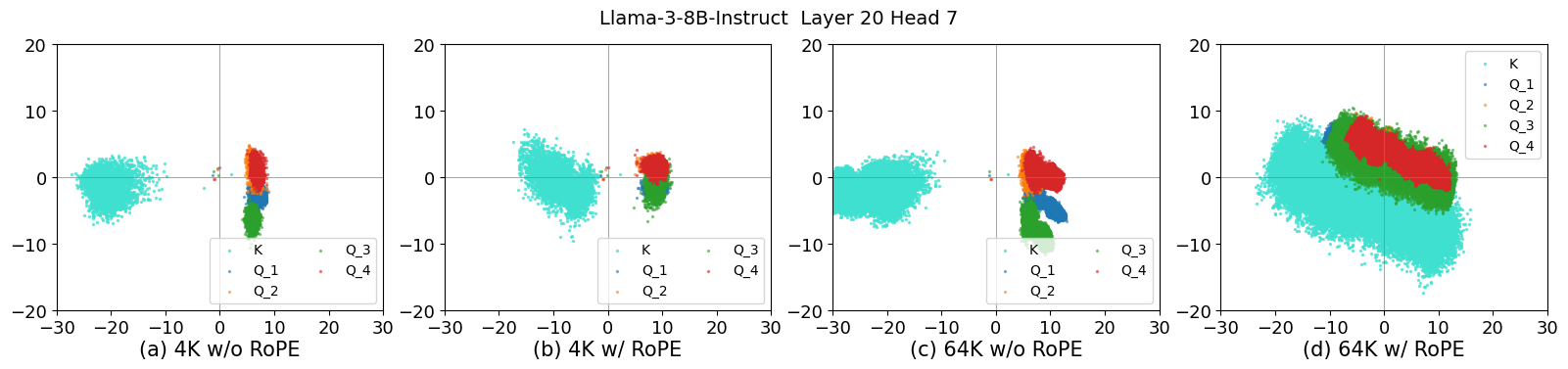}
\end{figure}

\vspace{-30pt}
\begin{figure}[h!]
    \centering
    \includegraphics[width=1.0\textwidth]{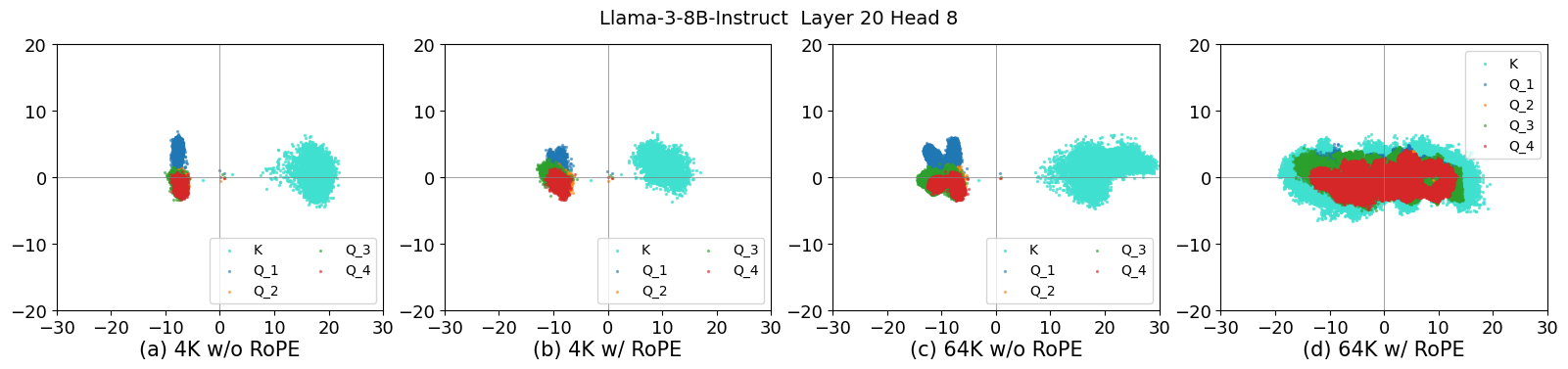}
\end{figure}

\vspace{-30pt}
\begin{figure}[h!]
    \centering
    \includegraphics[width=1.0\textwidth]{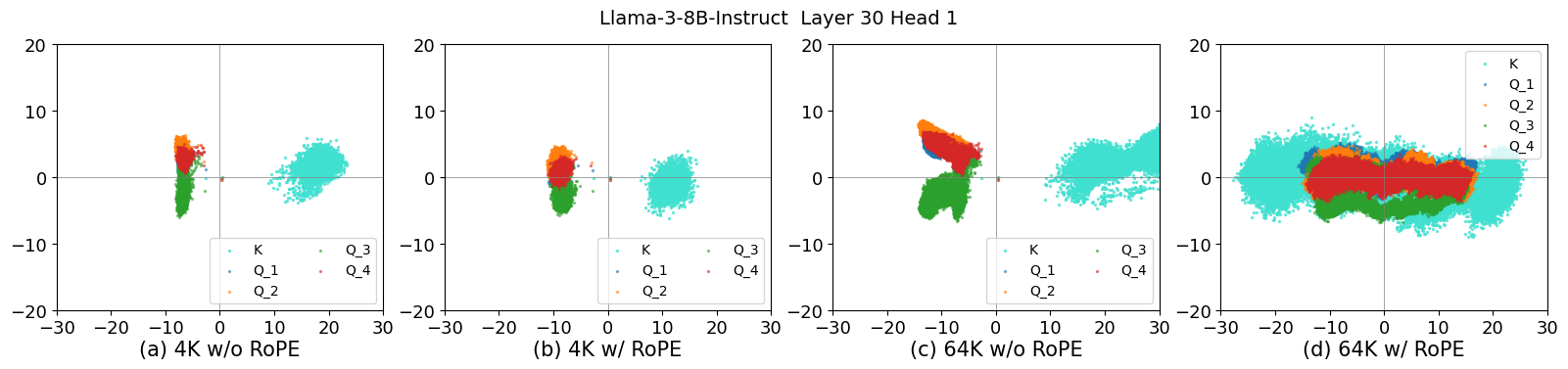}
\end{figure}

\vspace{-30pt}
\begin{figure}[h!]
    \centering
    \includegraphics[width=1.0\textwidth]{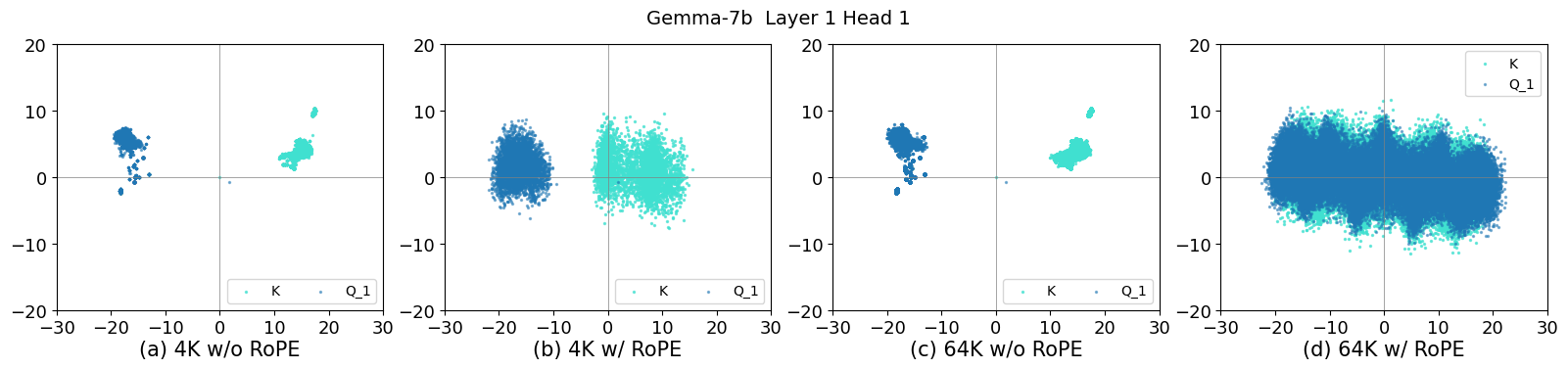}
\end{figure}
\vspace{-30pt}
\begin{figure}[h!]
    \centering
    \includegraphics[width=1.0\textwidth]{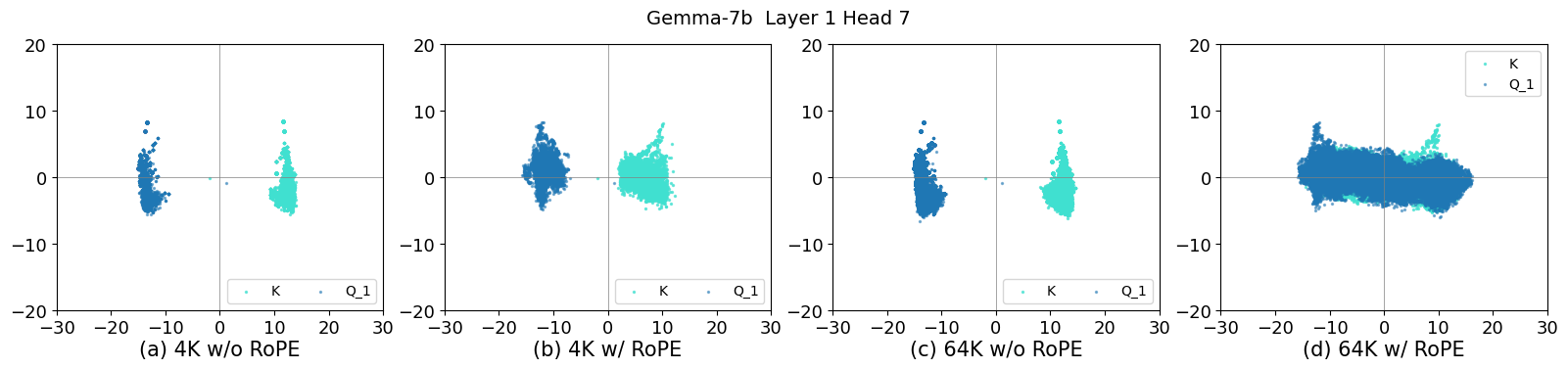}
\end{figure}

\vspace{-30pt}
\begin{figure}[h!]
    \centering
    \includegraphics[width=1.0\textwidth]{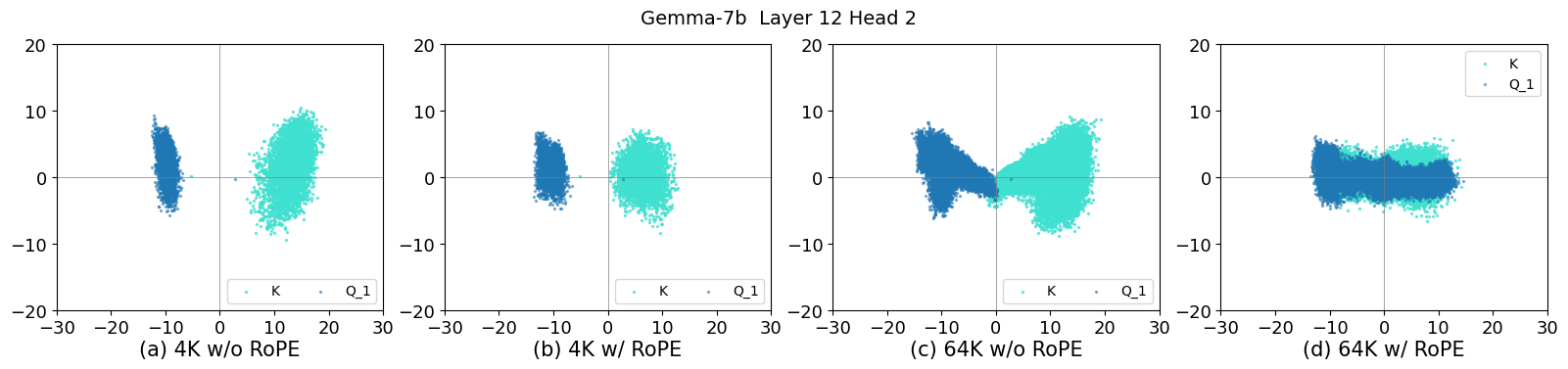}
\end{figure}

\vspace{-30pt}
\begin{figure}[h!]
    \centering
    \includegraphics[width=1.0\textwidth]{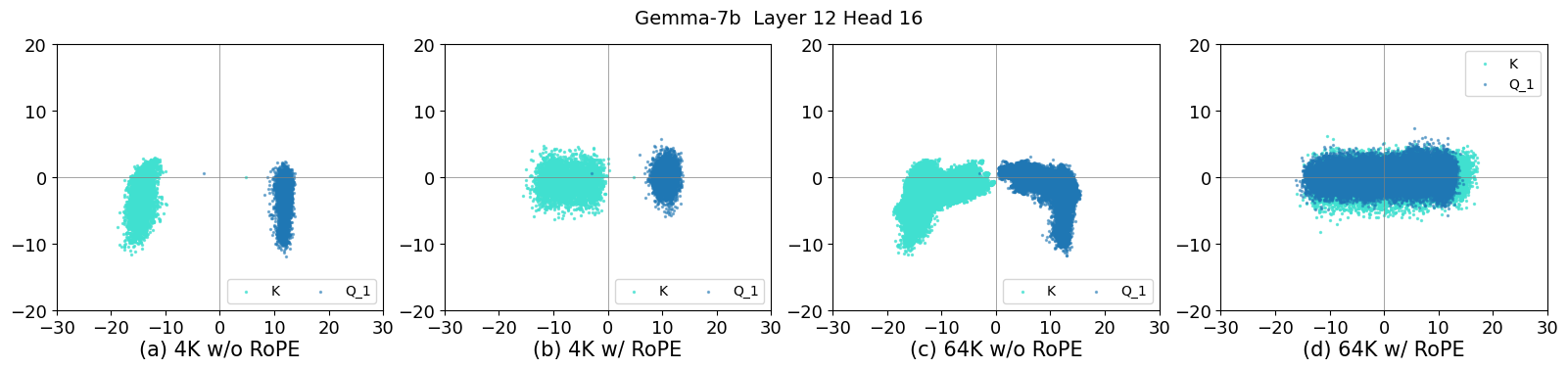}
\end{figure}

\vspace{-30pt}
\begin{figure}[h!]
    \centering
    \includegraphics[width=1.0\textwidth]{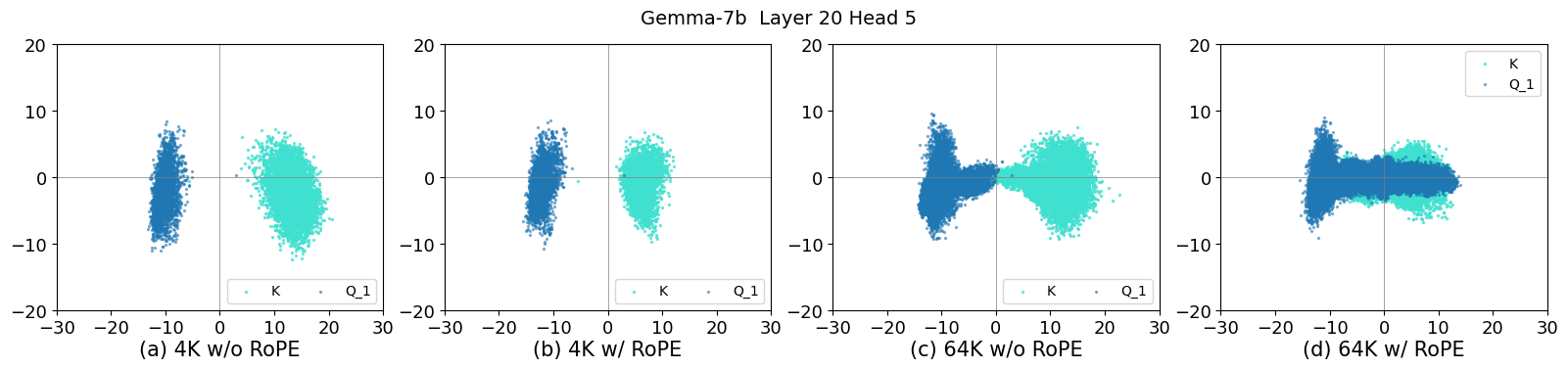}
\end{figure}

\vspace{-30pt}
\begin{figure}[h!]
    \centering
    \includegraphics[width=1.0\textwidth]{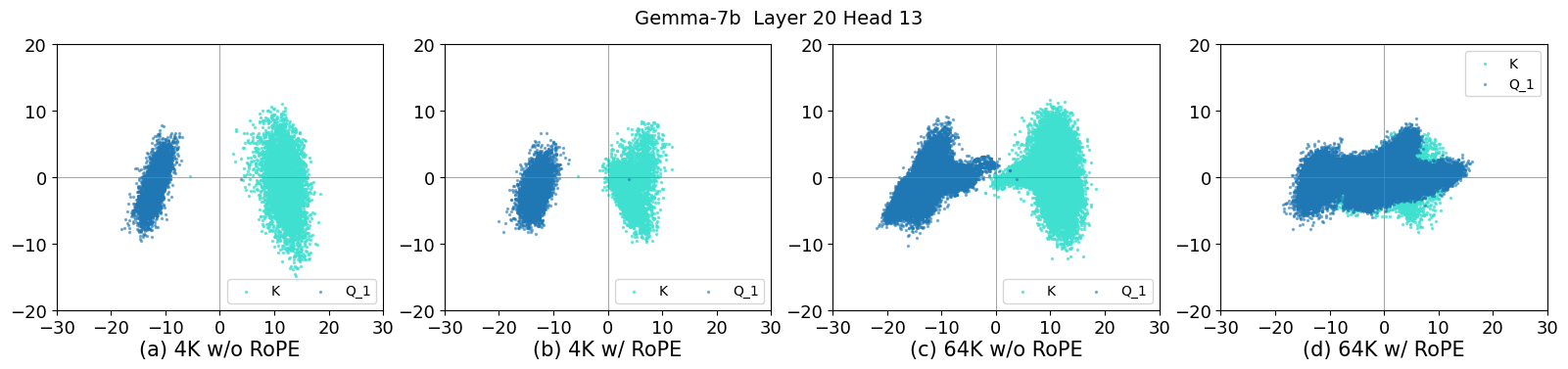}
\end{figure}

\vspace{-30pt}
\begin{figure}[h!]
    \centering
    \includegraphics[width=1.0\textwidth]{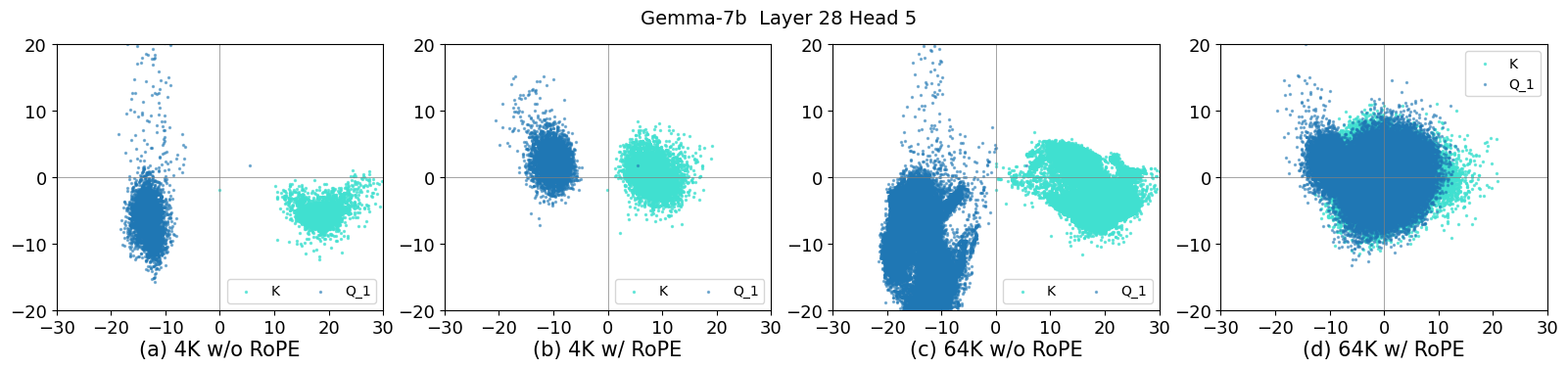}
\end{figure}

\vspace{-30pt}
\begin{figure}[h!]
    \centering
    \includegraphics[width=1.0\textwidth]{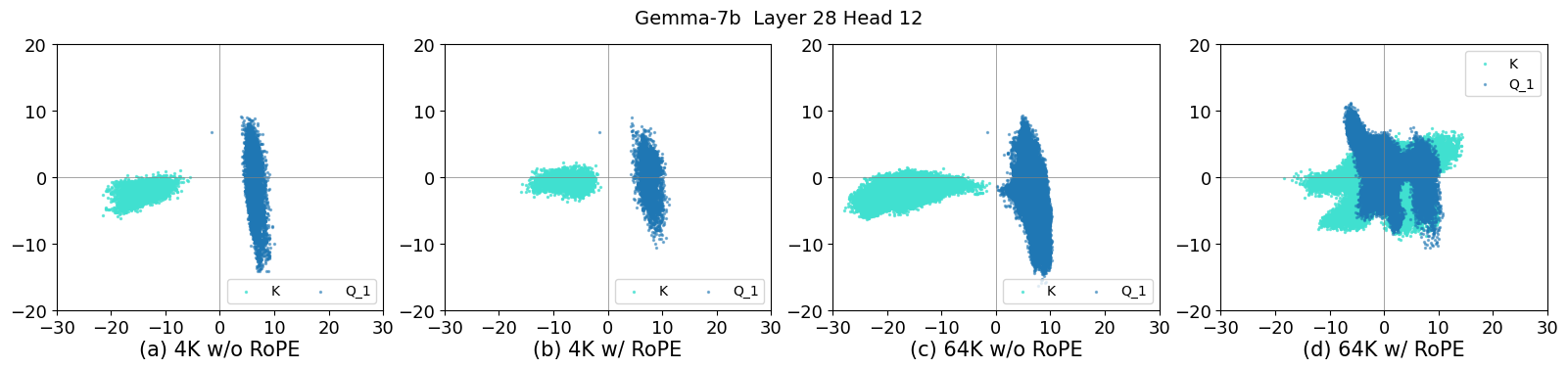}
\end{figure}

\vspace{-30pt}
\begin{figure}[h!]
    \centering
    \includegraphics[width=1.0\textwidth]{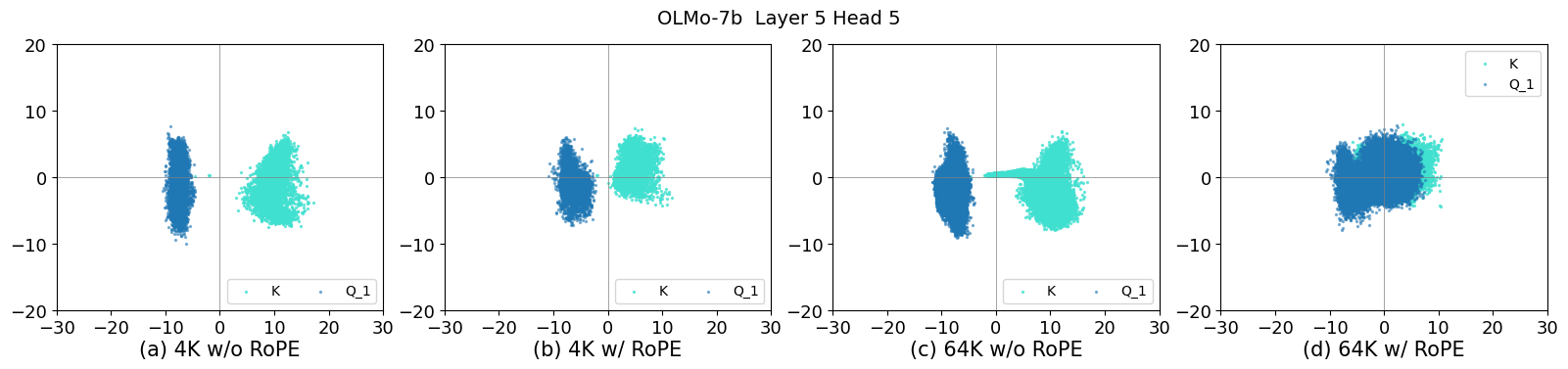}
\end{figure}
\vspace{-30pt}
\begin{figure}[h!]
    \centering
    \includegraphics[width=1.0\textwidth]{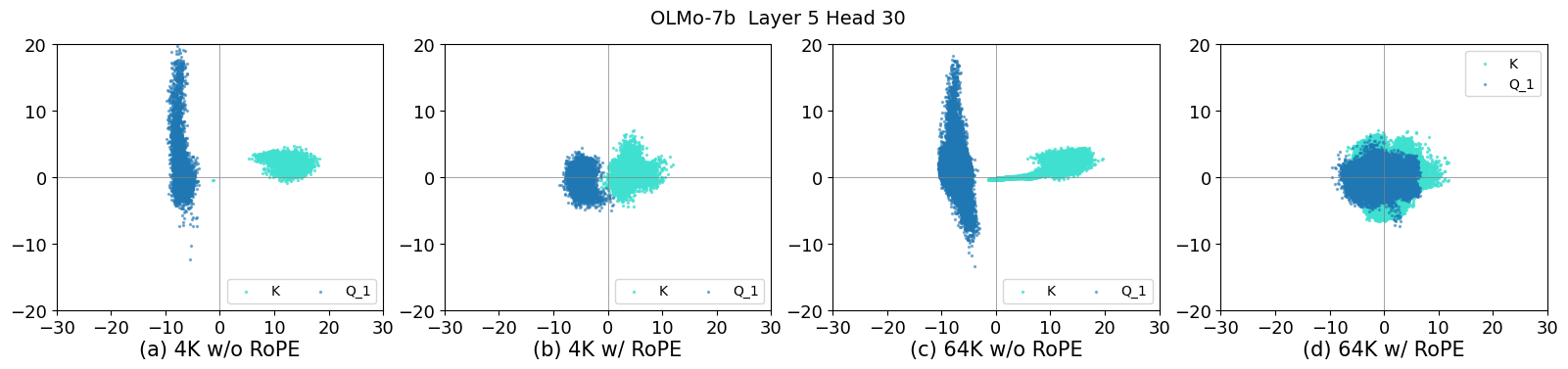}
\end{figure}
\vspace{-30pt}
\begin{figure}[h!]
    \centering
    \includegraphics[width=1.0\textwidth]{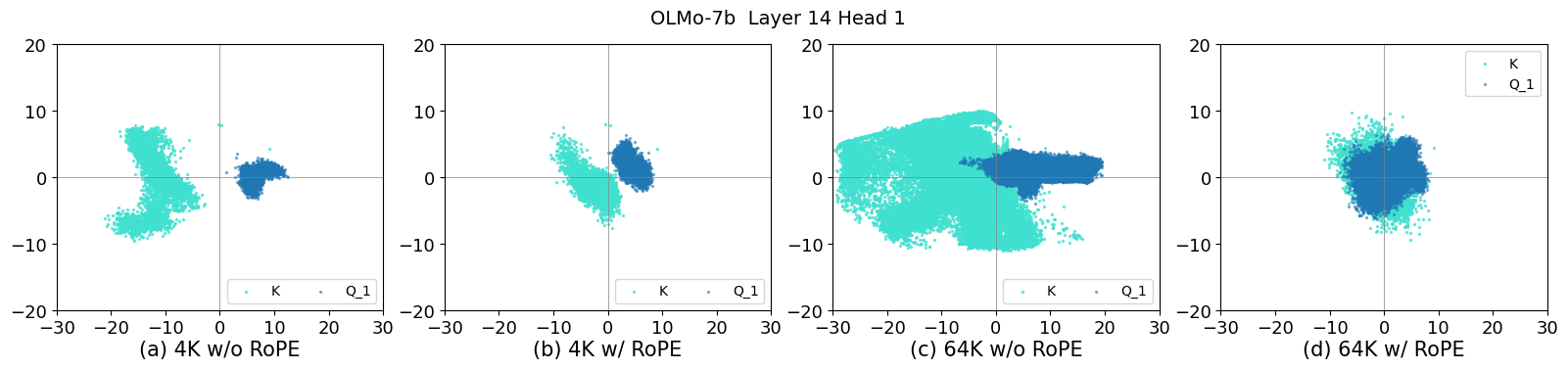}
\end{figure}

\vspace{-30pt}
\begin{figure}[h!]
    \centering
    \includegraphics[width=1.0\textwidth]{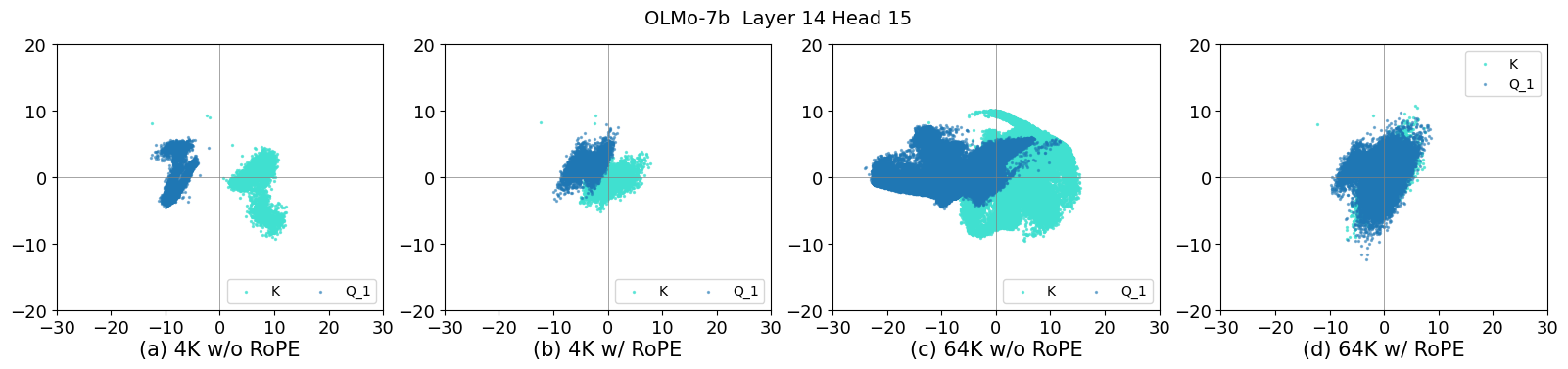}
\end{figure}

\vspace{-30pt}
\begin{figure}[h!]
    \centering
    \includegraphics[width=1.0\textwidth]{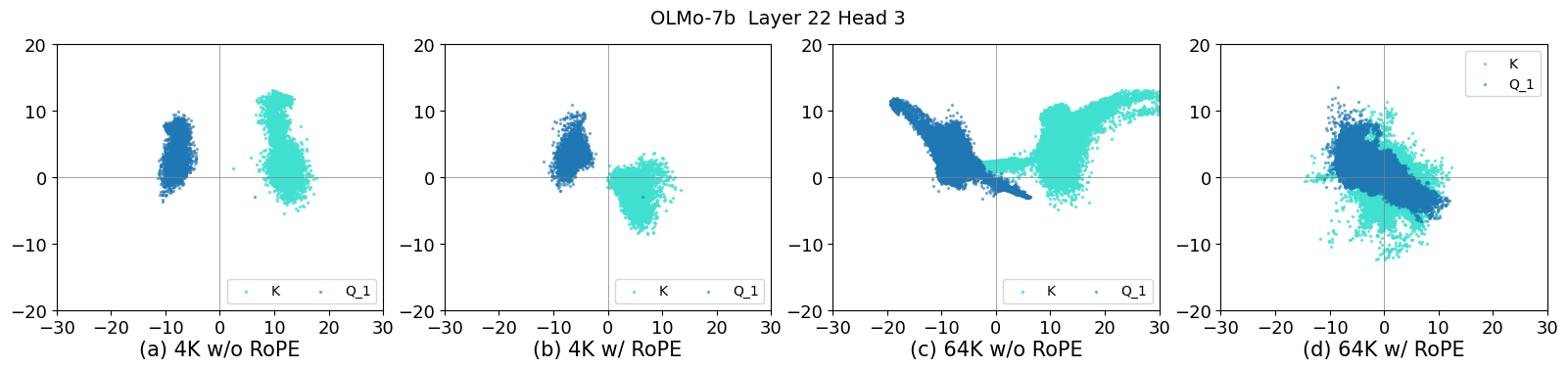}
\end{figure}
\vspace{-30pt}
\begin{figure}[h!]
    \centering
    \includegraphics[width=1.0\textwidth]{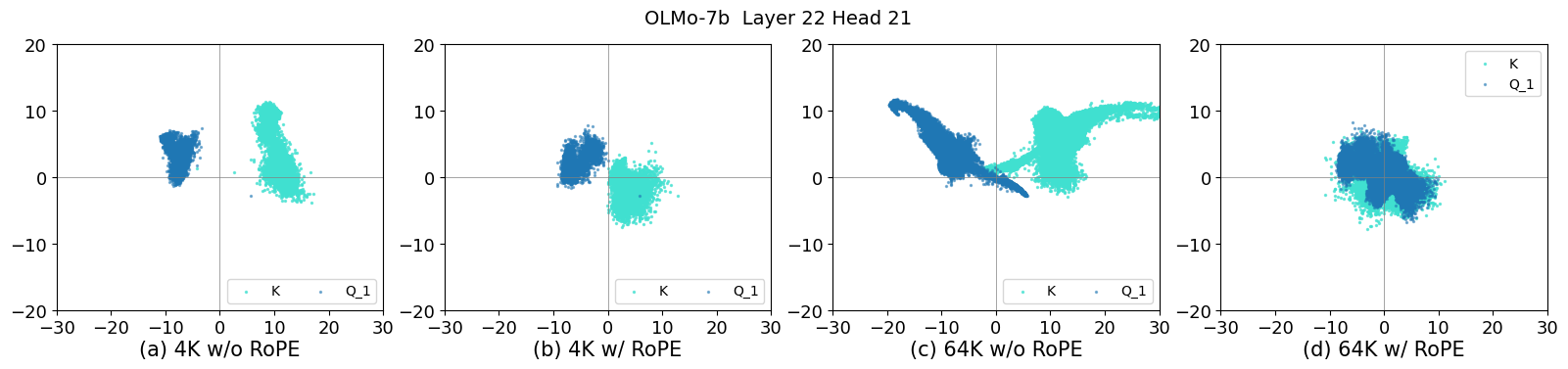}
\end{figure}

\vspace{-30pt}
\begin{figure}[h!]
    \centering
    \includegraphics[width=1.0\textwidth]{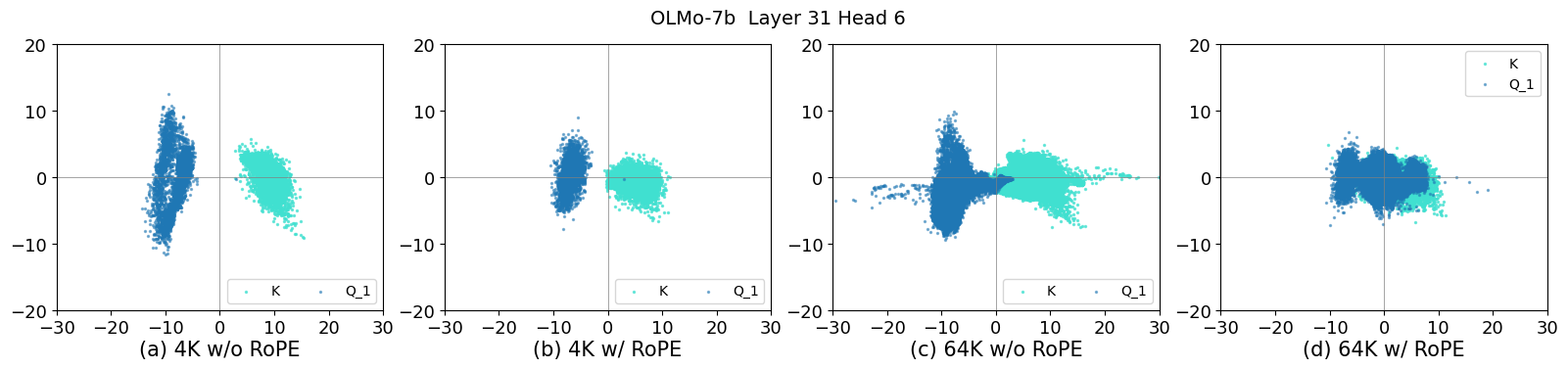}
\end{figure}
\vspace{-30pt}
\begin{figure}[h!]
    \centering
    \includegraphics[width=1.0\textwidth]{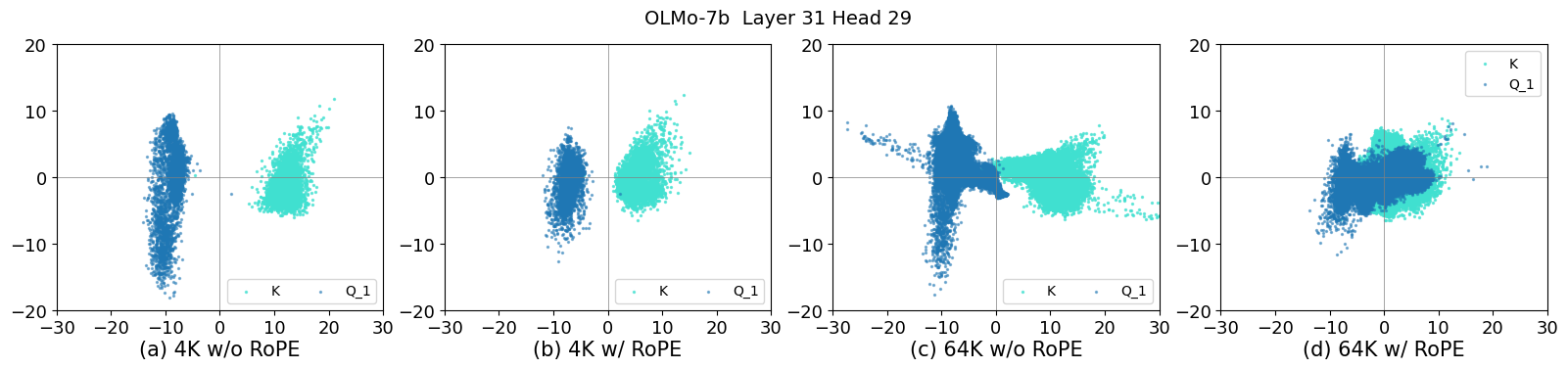}
\end{figure}

\end{document}